\documentclass[12pt]{article}
\usepackage{psfrag,amsmath,amssymb,latexsym,epic, eepicemu, epsfig,
 float, enumerate}
\usepackage{amscd }
\usepackage{amsfonts}
\usepackage{graphicx}
\usepackage{epsfig,psfrag,amsmath,amssymb,latexsym}
\usepackage{amscd}
\usepackage{xcolor}
\usepackage{amsfonts}
\usepackage{graphics}
\usepackage{graphicx}
\usepackage{float}
\usepackage[T1]{fontenc}
\newtheorem{theorem}{Theorem}[section]

\newtheorem{claim}{Claim}[section]

\newtheorem{lemma}{Lemma}[section]

\newtheorem{proposition}{Proposition}[section]

\newenvironment{proof}[1][Proof]{\textbf{#1.} }{\ \rule{0.5em}{0.5em}}

\newcommand{\e}{\epsilon}

\def\P{{\cal P}}
\def\p{\bar{\pi}}

\numberwithin{equation}{section}

\def\X{{\cal X}}
\def\e{\epsilon}
\def\bp{\bar{\pi}_n}
\def\a{\alpha}

\newcommand {\bx}{\mathbf{x}}

\title{An evolutionary model that satisfies detailed balance}

\author{J\"uri Lember\footnote{Institute of Mathematics and Statistics, University of Tartu, Estonia; jyril@ut.ee; +372 7375487}, Chris Watkins\footnote{Department of Computer Science, Royal Holloway, University of London, UK}}
\begin{document}
\maketitle
\begin{abstract}We propose a class of evolution  models that involves an arbitrary exchangeable process as the breeding process and different selection schemes.
In those models, a new genome is born according to the breeding
process, and after that  a genome is removed according to the
selection scheme that involves fitness. Thus, the  population size
remains constant. The process evolves according to a Markov chain,
and,  unlike in  many other existing models,  the stationary
distribution -- so called mutation-selection equilibrium -- can
easily found and studied. As a special case our model contains a (sub)class of Moran models. The behaviour of the  stationary
distribution when the population size increases is our main object
of interest. Several phase-transition theorems are proved.
\end{abstract}
\section{Introduction}
\label{sec:probabilitymodels}

We introduce a probability model of evolution that has three purposes: as an abstract model of biological evolution; as a class of efficiently implementable genetic algorithms that are easy to analyze theoretically; and as a link between genetic algorithms and Bayesian non-parametric MCMC methods. The stationary distribution of the model can be expressed in closed form for arbitrary fitness functions: this enables us to investigate the behaviour of the model for different population sizes, mutation rates, and fitness scalings. We find two phase transitions that occur for all fitness functions.  Our approach is most applicable to evolution by sexual rather than asexual reproduction.

In any model of evolution under constant conditions, there is a
temporal sequence of possibly overlapping populations. In the
transition from each population to the next, one or more new
individuals are `born', and the same number of individuals are
removed from the population, or `die'. Each new generation is a
population that depends only upon the previous parent population, so
that the sequence of populations is a Markov chain.  In a model with
mutation, the Markov chain is irreducible and has a unique
stationary distribution that is also known as the mutation-selection
equilibrium. Typically we wish to know what the stationary distribution is.

{For Moran models with only a single locus, the stationary distribution can be characterised.
But for models with more than one locus, and fitness that depends nonlinearly on the combination of alleles at the loci, the stationary distribution is notoriously hard to characterise.  The reason  is
that in previous models of sexual reproduction with multiple loci, the
Markov chain of populations is irreversible, so that there is no
obvious method of finding the stationary distribution other than
attempting to compute eigenvectors of the transition matrix
directly, as done in \cite{vose1999simple}, but these calculations
are neither easy nor revealing for arbitrary fitness functions on multiple loci. }

For a reversible Markov chain, the stationary distribution may be found by verifying detailed balance conditions.
In our model, introduced in  \cite{watkins2014sex}, we start by writing the mutation-selection equilibrium in a convenient closed form, and then
exhibit MCMC kernels that implement reversible Markov chains with this stationary distribution, and for which the proposal and acceptance algorithms are plausible abstract models of breeding, mutation, and selection. Each generation starts with a population of $n$ individuals. One new
individual is `bred' -- that is, sampled conditionally on the
existing population -- to produce an expanded population of $n+1$
individuals; from this expanded population, one individual is
selected to be discarded, leaving a new population of size $n$ to
start the next generation.  This kind of evolution  is typically modeled with a  Moran model
\cite{moran1962statistical} and in subsections \ref{sec:dir} and \ref{sec:moran}, we shall see that when the breeding process is a generalised Polya urn model, then our model can be indeed considered as a discrete time Moran model with mutation and selection. But our model is more general that just a Moran one, because we allow a more general breeding process. However, due to the connection with Moran models, the generalized Polya urn model (or, equivalently, Dirichlet-categorical process) is an important special case of breeding so that section \ref{sec:dirichletprior} is devoted to that model.

In many evolutionary models such as \cite{crow1970introduction}, and in genetic algorithms such as \cite{holland1975adaptation},
the reproductive fitness of a genome is modelled as the genome's rate of breeding, and not as its probability of death. In these  models, in the breeding
phase of each generation, fit genomes are chosen more frequently to breed than unfit genomes are: in these models,  discarding individuals -- or `death' -- is modelled as uniformly random
deletion from the population. The second option is that the selection occurs at death: individuals are discarded (or `die')
with probability related to their fitness (termed viability selection): less-fit individuals are more likely to be discarded. For discrete time classical models such as the Wright-Fisher or Moran model, there is actually no big difference whether the selections occur at death or at reproduction (although formulas can be slightly different, see e.g. \cite{MuireheadWakeley}) and in the literature both versions are present.
In our model, breeding is conditional sampling from a very general exchangeable process and so it is quite natural that selection occurs at death.
This is a significant design choice that turns out to greatly simplify the analysis.
Observe that selection at death it does not mean that all allele  types have equal probability to be born. Our model might have breeding process as i.i.d. random variables with given distribution 
possibly very far form uniform, or the breeding process can be generalized Polya urn, where the history of evolution determines what types are more likely to be born.

As explained in detail in sections \ref{sec:breeding} and \ref{sec:selection}, careful choices in modelling breeding as conditional sampling, and in modelling fitness by
stochastic rejection of the less fit, enable us to construct a reversible Markov chain of populations that is a Metropolis-Hastings process.

\paragraph{Mathematical formulation, overview of the main results and
organization of the paper.} Throughout the paper, we suppose there
is a finite set $\X=\{1,\ldots,K\}$ of possible genomes. Our models apply both to the case of genomes with one locus with $K$ possible alleles,  and also to the case where genomes have more than one locus: for example, for genomes with $L$ loci with two possible alleles at each, then $K=2^L$.  Throughout the paper, $\xi_1,\xi_2,\ldots$
denotes an exchangeable $\X$-valued stochastic process. For any
population of $n$ genomes $\bx = (x_1, \ldots, x_n) \in \X^n$, we
define
\begin{equation}\label{eq:subtledef}
P_\xi(\bx) :=  P_{\xi}(x_1,\ldots,x_n) :=  P(\xi_1=x_1,\ldots,\xi_n=x_n)
\end{equation}
By definition of exchangeability, we have that for any permutation $\sigma$,
\begin{equation*}
P_\xi(x_1, \ldots, x_n) = P_\xi( x_{\sigma(1)}, \ldots , x_{\sigma(n)} ).
\end{equation*}
Given a population of genomes $\bx = (x_1,\ldots,x_n)$, we breed an $n+1$'th genome by sampling $x_{n+1}$ conditionally:
\begin{equation*}
x_{n+1} \sim P_\xi(\cdot \mid x_1, \ldots, x_n).
\end{equation*}
The `fitness' of a genome $x$ is denoted by $w(x)>0$, where $w$ is
an arbitrary strictly positive function over $\mathcal{X}$. In the
evolutionary models we propose, in each generation one individual is
`bred' by conditional sampling from the existing population, and
then an individual is discarded in a fitness-biased way, so that
less-fit individuals are more likely to be discarded.
As explained in subsections
\ref{sec:singletournamentselection},
\ref{sec:inversefitnessselection} and \ref{sec:breedmany}, there are
many possible schemes to formalize this idea, and we shall show that
all those schemes lead to one particular stationary distribution of
populations that factorises into the form:
\begin{equation}\label{eq:stationary1}
\underbrace{P_n(x_1, \ldots , x_n )}_\text{stationary distribution}   =\frac{1}{Z_n}  \underbrace{P_{\xi}(x_1, \ldots, x_n)}_\text{breeding term} ~  \underbrace{w(x_1) \cdots w(x_n)}_\text{fitness term}.
\end{equation}
The process is similar to, but not the same as non-parametric Bayesian MCMC, and we make this connection explicit in section \ref{sec:BayesianMCMC}.

The measure $P_n$ is the main object of
interest. Since $\xi$ is exchangeable, there exists a prior measure
$\pi$ on the  set of all probability measures on $\X$ (called simplex),
denoted by  $$\mathcal{P}:=\{(q_1,\ldots,q_k): q_i\geq 0, \sum_i q_i=1\},$$ such that for every $\bx\in \X^n$
\begin{equation}\label{fi}
P_{\xi}(\bx)=\int_{\P}\prod_{i=1}^n q(x_i)\pi(dq).
\end{equation}
Here $q(k)=q_k$, in what follows, we shall use both notations ($q_k$ and $q(k)$). Hence $P_n$ is fully determined by prior measure $\pi$ and fitness
function $w$. In order to achieve more generality, we allow $w$ and
$\pi$ to depend on population size $n$ (thus writing $w_n$ and
$\pi_n$), and we  ask: do the  measures $P_n$ converge (in some
sense) to a limit? The sense of convergence needs to be specified,
because  $P_n$ is defined on $\X^n$, so that the domain of the
measure depends on $n$. Therefore we consider the array of random
variables $(X_{n,1},\ldots,X_{n,n})\sim P_n$ and we ask: is there a
stochastic process $X_1,X_2,\ldots$ so that for every $m$
\begin{equation*}(X_{1,n},\ldots X_{m,n})\Rightarrow
(X_{1},\ldots X_{m})?\end{equation*}
Another option to define the limit is based on the observation that
the measure $P_n$ is invariant with respect to the permutations
(finitely exchangeable), and so it depends on the counts of vector
${\bf x}$, only. This, in turn, allows to define a counterpart of
$P_n$ -- a probability measure $Q_n$ on the simplex $\P$. The formal
definition of $Q_n$ is given in subsection \ref{sec:kuu}. Since all
measures $Q_n$ are defined on the same domain (simplex $\P$), we can
now ask whether the sequence $Q_n$ converges in the classical sense
of weak convergence of probability measures. Throughout the paper,
these two approaches -- measures $P_n$ (random process) and measures
$Q_n$ -- are handled in parallel.

Our first limit result, Theorem \ref{thm1} considers
the case when the prior measure $\pi$ is independent of $n$, but
arbitrary; the fitness function $w_n$ depends on $n$ in the
following way (recall $\X=\{1,\ldots,K\}$):
\begin{equation}\label{wn}
w_n(k)=\exp\left[-{\phi(k)\over n^{\lambda}}\right], \quad k=1,\ldots,K
\end{equation}
 where $\lambda\geq 0$
and $0\leq \phi(1)\leq \phi(2)\leq \cdots \leq \phi(K)$. The case $\lambda=0$
corresponds to fitness that is constant in that it does not vary
with $n$.  Observe that for any $n$, the genotype 1 is the fittest.
Theorem \ref{thm1} shows that the following phase transition occurs
with respect to $\lambda$:
\begin{enumerate}
  \item when $\lambda\in [0,1)$ and $\phi(1)<\phi(2)$, then  the limit process
  $X_1,X_2,\ldots$ has one trajectory $1,1,1,\ldots$ a.s., and
  $Q_n\Rightarrow \delta_{q^*}$, where $q^*=(1,\ldots,0).$
  Thus, when  $\lambda\in [0,1)$   only the fittest genotype
survives in the limit,  no matter what the prior
says; the influence of the prior vanishes and
fitness prevails. The additional assumption $\phi(1)<\phi(2)$ means that $w_n(1)>w_n(k)$, $k=2,\ldots,K$, and so it is quite natural for the result.
 \item When $\lambda=1$,  $Q_n$
converges to a nondegenerate distribution specified in Theorem
\ref{thm1} that depends on the prior $\pi$ and on the function
$\phi$.  When the prior $\pi$ has density  $\pi(q)$, then the limit distribution also has density proportional to 
\begin{equation}\label{l1}
\pi(q)\exp[-\langle \phi,q \rangle].
\end{equation}
Then also the limit process $X_1,X_2,\ldots$ is exchangeable process with prior (\ref{l1}).
non-degenerate.
\item When $\lambda>1$, then the influence of fitness vanishes and
only the prior matters -- the limit process is equal to the
breeding process and $Q_n\Rightarrow \pi$.
\end{enumerate}
Section \ref{sec:dirichletprior} considers Dirichlet priors
\begin{equation}\label{dir}
\pi_n={\rm Dir}(n^{1-\lambda}\alpha_1,\ldots
n^{1-\lambda}\alpha_K),\end{equation} where
$\a:=(\alpha_1,\ldots,\a_K)$, $\alpha_k>0$. Thus $\pi_n$ admits a density on ${\cal P}$
$$\pi(q_1,\ldots,q_k)\propto \prod_{k=1}^K q_k^{n^{1-\lambda}\alpha_k}.$$
With a Dirichlet prior the
breeding process is  a (generalized) Polya urn process, also known
as a Dirichlet-categorical process. This process is a natural
choice because it admits a notion of `mutation'.  Conditional
sampling from it is directly interpretable as a simplified model of
sexual breeding with mutation, as explained in section
\ref{sec:breeding}. The  fitness is previously
defined as in (\ref{wn}) and the common parameter $\lambda$ allows
to interpolate between influence of fitness and prior. The case
$\lambda=1$ corresponds to the constant prior case. Our second main
limit theorem, Theorem \ref{thm2} shows that with respect to
$\lambda$, phase transition occurs again:
\begin{enumerate}
  \item when $\lambda=0$, then the limit process $X_1,X_2,\ldots$
  consists of i.i.d. random variables with certain distribution
  $r^*$ (specified in Theorem \ref{thm2}) depending on $\phi$
  and $\alpha$, and $Q_n\Rightarrow \delta_{r^*}$.
  \item when $\lambda\in (0,1)$, then then the limit process $X_1,X_2,\ldots$
  consists of i.i.d. random variables with another distribution
  $q^*$ (specified in Theorem \ref{thm2}) also depending on
  $\phi$ and $\alpha$, and $Q_n\Rightarrow \delta_{q^*}$;
  \item when $\lambda=1$ and $\phi(1)<\phi(2)$, then the limit is specified by Theorem
  \ref{thm1}. From (\ref{l1}) it follows that the limit distribution of allele proportions has density proportional to
  \begin{equation}\label{l1dir}
   \exp[-\langle \phi,q \rangle]\prod_{k=1}^K q_k^{\alpha_k-1}.
   \end{equation}
  The limits of type (\ref{l1dir}) are common for standard models in evolution theory, both for Wright-Fisher and Moran model (see, e.g.  \cite{Durrett,Ewens,Feng}). We shall discuss the connection with our and Moran model in subsections \ref{sec:dir} and \ref{sec:moran}.

\end{enumerate}
The phase transitions at $\lambda=0$ and $\lambda=1$ are sharp. These are the most
important results of the paper.

The paper is organized as follows.  In section \ref{sec:breeding} we
give examples of exchangeable conditional sampling procedures that
can be regarded as abstract models of breeding, and section
\ref{sec:selection} gives examples of selection procedures that,
together with any of the breeding procedures, will
 produce a reversible Markov chain of populations with the stationary distribution of  (\ref{eq:stationary1}).  Subsection \ref{sec:noiseinvariance} establishes
 that the stationary distributions are invariant to multiplicative fitness noise; subsection \ref{sec:BayesianMCMC} shows that special cases of this process are forms of
  Bayesian inference by MCMC.  Next, section \ref{sec:measurepn} establishes basic conditions on the convergence of the stationary distribution to a limit distribution as population size tends to infinity.
  Section \ref{sec:largepopulationlimit} is devoted
  to Theorem \ref{thm1} and in section
  \ref{sec:dirichletprior} generalized Polya urn
  processes (Dirichlet priors) are considered. In this section,
  Theorem \ref{thm2} is proved and the product of independent Polya urn processes are considered (subsection
  \ref{sec:prod}).
    We present computational
  experiments demonstrating our results in section \ref{sec:experiments}. Finally we discuss implications of our results for genetic algorithms and evolutionary modelling in
  \ref{sec:conclusions}.

\section{Breeding and mutation}
\label{sec:breeding}

We model breeding as Gibbs sampling \cite{geman1984stochastic} from
an exchangeable distribution; exchangeable Gibbs sampling is a
standard technique  in statistical nonparametric MCMC methods, for
example \cite{neal2000markov,hjort2010bayesian}, but in that context
it is not of course regarded as a model of breeding. The property of
exchangeability of $P_\xi$ will be used in two ways: first, in
section \ref{sec:selection} we will use it to establish detailed
balance for several selection procedures, which establishes that the
stationary distribution of populations is indeed as given in
 (\ref{eq:stationary1}). Second, in section
\ref{sec:measurepn} we use the de Finetti integral representation of
$\xi$ to establish limit properties of the stationary distribution
as $n\rightarrow\infty$. We now give examples of  conditional
sampling that can be regarded as plausible models of breeding.
\subsection{Dirichlet-Categorical  Process}\label{sec:dir}
In the simplest case, each `genome' consists of only one locus or `gene'
which can be one of $K$ possible allele `types', that we denote by
$\{1,\ldots,K\}$. The exchangeable process $\xi$ is the well known
Polya urn model for a Dirichlet process with discrete base
distribution. We recall the definition of this process. Let $\alpha
= (\alpha_1, \ldots, \alpha_K)$, $\alpha_i>0$ be the prior
parameters of the base distribution; we write $|\alpha| := \alpha_1
+ \cdots + \alpha_K$. Let $\xi_1, \xi_2, \ldots$ be a random process
over $\{ 1,\ldots,K\}$. Let $n\ge 0$. Given $\bx=x_1, \ldots, x_n$,
 we denote the number of $k$-s in the sequence by
$n_k(\bx)$:
\begin{equation}\label{fr}
n_k=\sum_{i=1}^n I_{\{k\}}(x_i).
\end{equation}
From the Polya urn,  the next allele is drawn of type $k$ with the following probability:
\begin{equation*}
P_\xi( \xi_{n+1} = k \mid x_1, \ldots, x_n, \alpha ) := \frac{n_k + \alpha_k}{n + |\alpha | }.
\end{equation*}
It follows that
\begin{equation*}
P_\xi( x_1, \ldots, x_n \mid \mathbf{\alpha} ) = {(\alpha_1)_{n_1}\cdots (\alpha_K)_{n_K}\over (|\alpha|)_n},
\end{equation*}
where for any $\alpha>0$
$$(\alpha)_n:=\alpha(\alpha+1)\cdots(\alpha + n-1).$$
{The process}  $\xi_1, \xi_2, \ldots$ is infinitely exchangeable by inspection. Due to the exchangeability, by the formula above, it is clear that the random vector of counts has the distribution
\begin{equation}\label{cnts}
P_{\xi}(n_1,\ldots,n_K)={n!\over n_1!\cdots n_K!}{(\alpha_1)_{n_1}\cdots (\alpha_K)_{n_K}\over (|\alpha|)_n}.
\end{equation}
By de Finetti's theorem
$$P_\xi( x_1, \ldots, x_n \mid \mathbf{\alpha} ):=\int_{{\cal P}}
q(x_1)\cdots q(x_n)\pi (dq),$$ where the prior measure $\pi$ is Dirichlet distribution,
i.e. $\pi={\rm Dir}(\a_1,\ldots,\a_K)$. This process is in a sense
the central object of our study.

\paragraph{Concentration parameter, mutation rate and Moran models.}
The concentration parameter ${|\alpha|} = \alpha_1 + \cdots +
\alpha_K$ may be viewed as determining a mutation rate that depends
on $n$. With $n$ balls in the {sample (thus $n+|\a|$
balls in the urn)} when a new ball is sampled, there is a
probability $\frac{|\alpha|}{n+|\alpha|}$ that the ball will be
sampled from the collection of `prior' balls, rather than the actual
balls in the urn. This probability is independent of the colors of the
$n$ actual balls present:  since a new color may be introduced in
this way, we regard this as analogous to a mutation. The mutation
rate $u$ is
\begin{equation*}
u := \frac{|\alpha|}{|\alpha| + n} ~~~ \text{or equivalently} ~~~ |\alpha| = \frac{nu}{1-u}
\end{equation*}
To be more detailed, for any type $k=1,\ldots,K$, the `prior' ball of type $k$ is sampled
with probability
\begin{equation}\label{ua}
u_k:={\alpha_k\over |\alpha|+n} ~~~ \text{or equivalently} ~~~ \alpha_i = \frac{nu_i}{1-u}.\end{equation}
Hence $u_1+\cdots+u_K=u$ and $u_k$ is the rate (probability) of mutating from any other type to $k$. In our evolutionary processes, one new genome is sampled, and one genome is then discarded at each generation, so that $n$ remains constant. To define processes with the same mutation rate $u$ but with different values of $n$,
$\alpha$ must be adjusted to depend on $n$. Such single birth-death changes are often modeled via the well known Moran model, and it  turns out that our model of breeding  in some sense generalizes the drift-mutation  Moran model. Let us explain that connection more closely.
Recall that a standard (simple) $K$-allele Moran model without mutation and selection is as follows.  We consider a population of $n$ individuals and suppose  that at time $t$, the number of type $k$ individuals is $n_k$. An individual is randomly chosen to reproduce and another is independently chosen to die. When the type $i$ is chosen to reproduce and type $j$ is chosen to die, then
the allele counts at time $t+1$ are $(n_1,\ldots,n_i+1,\ldots,n_j-1,\ldots n_K)$ and the probability of such a change is ${n_in_j\over n^2}$.
There are several ways to incorporate mutation into the Moran model, but in our model the probability that type $i$ is born is
\begin{equation}\label{mor}
{n_i+\alpha_i\over |\alpha|+n}={n_i\over n}{n\over |\alpha|+n}+{\alpha_i\over |\alpha|+n_i}={n_i\over n}(1-u)+u_i\end{equation}
In the sum above, the first part is the transition probability without mutation multiplied with the probability of no mutation, the second part is the probability that the selected type mutates to type $i$. When the probability of $j$-type to die is just ${n_j\over n}$ we end up with a Moran model, where
the probability that $n_i$ becomes $n_i+1$ and $n_j$ becomes $n_j-1$ is
\begin{equation}\label{transition}{n_i n_j\over n^2}(1-u)+{n_j\over n}u_i.\end{equation}
It is easy to verify that  a Markov chain with transition (\ref{transition}) satisfies the detailed balance equation with stationary distribution being (\ref{cnts}), where $\alpha_k$ is defined as in (\ref{ua}). There are surely other ways to specify mutation in Moran model so that  (\ref{cnts}) is the stationary distribution. For $K=2$, one such example is in \cite{Ewens}: let $u_1$ be the probability that type $2$ mutates to 1 and $u_2$ be the probability that type 1 mutates to 2. Then the transition from $(n_1,n_2)$ to $(n_1+1,n_2-1)$ has probability
\begin{equation}\label{tr2}
{n_1n_2\over n^2}(1-u_2)+{n_2^2\over n^2}u_1,
\end{equation}
and detailed balance equation with (\ref{cnts}) (with $\alpha_k$ being defined as in (\ref{ua})) is easy to verify (this is (3.58) in \cite{Ewens}). Other examples of Moran models without selection having (\ref{cnts}) as stationary distribution can be found in \cite{VoglClemente}, where the authors consider a discrete time  version of so-called two-allele decoupled Moran model introduced in \cite{BaakeBialowons} (see also \cite{SchrempfHobolth}). In their model, the transition probability from  $n_i$ and $n_j$ to $n_i+1$ and $n_j-1$ without selection is very close to (\ref{transition}):
\begin{equation}\label{decoupled}
 {n_i n_j\over n^2}+u_i{n_j\over n}
\end{equation}
and the the detailed balance equation with (\ref{cnts})  is easy to to verify. In this case, $\alpha_i=nu_i$. A $K$-allele version of that model with transition matrix as in (\ref{decoupled}) is studied in \cite{EtheridgeGriffiths} and the stationary distribution (14) in \cite{EtheridgeGriffiths} is exactly (\ref{cnts}). Hence our breeding model largely incorporates mutation-drift version of Moran models, but it is much general, because the Dirichlet distribution is only one possible choice of prior measure $\pi$ in (\ref{fi}).

Note that in our model, mutations only occur at birth, and --
importantly -- the distribution of mutations does not depend on the
frequencies of different colors that are currently in the
population; mutations are distributed according to a fixed prior
distribution determined by the frequencies of `prior' balls in the
urn.

In our model, we count as a mutation any ball which results from a
draw of a `{prior}' ball: this definition is consistent with an
extension of our model to Dirichlet processes with continuous base
distributions, which we intend to consider in future work.
\subsection{Complex genomes: direct product of Dirichlet processes}
\label{sec:product}
More complex evolutionary models such
as genetic algorithms  require more complex genomes. Suppose that
each genome is a vector of $L$ genes, $x_i = (y^1_i , \ldots ,
y^L_i)$. Let $\xi$ be a direct product of $L$ independent
exchangeable processes $\xi = ( \xi^1, \ldots, \xi^L )$. Then
\begin{equation}
P_\xi( x_1 , \ldots, x_n) := \prod_{j=1}^L P_{\xi^j}(y^j_1, \ldots, y^j_n)
\end{equation}
which clearly is exchangeable as well. Exchangeable sampling from
$P_\xi(\cdot \mid x_1, \ldots, x_n)$, where each $x_i$ is a vector
of discrete values, can be viewed as a model of sexual reproduction
with the assumption of linkage equilibrium. A new vector $(y_{n+1}^
1, \ldots, y_{n+1}^L)$ is sampled by, for $1\le j \le L$, sampling
the $j$-th component $y_{n+1}^j \sim P_{\xi^{j}}(\cdot \mid y^j_1,
\ldots, y^j_n)$ independently from the rest of the components. In
words, each new element of the vector $x_{n+1}$ is either a copy of
the corresponding element of a randomly chosen member of the
existing population, or else a mutation. Instead of a new `child'
genome being constructed by random recombination of two parent
genomes, it is instead a random recombination of all $n$ existing
genomes in the population, with mutations.

This method
of constructing new genomes by `$n$-way recombination' is a widely used approach in genetic algorithms,
as used by
\cite{baluja1995removing,baluja1997genetic,baum2001genetic} and
others, and sexual reproduction with full linkage equilibrium is a
standard simplified model of sexual reproduction in population
genetics theory \cite{crow1970introduction,Ewens}.

The extension to Cartesian products of Dirichlet processes might
appear rather simple because each component of a new genome is
sampled independently of the others; however, this extension can
lead to models of great complexity because the fitness function $w$,
or equivalently $\phi$, can be an arbitrary function on
$\mathcal{X}^L$, so that the stationary distribution need not be a
product distribution. In genetic language, the fitness function can
have arbitrary epistasis.

Note that there are other discrete exchangeable distributions based
on Dirichlet distributions that could also be used as $P_\xi$. A
notable example is the discrete fragmentation-coagulation sequence
process introduced in \cite{elliott2012scalable}; this was intended
as statistical model for imputing phasing in genetic analysis, but
it could also be used as a breeding distribution for our purposes.

\section{Selection}
\label{sec:selection}

Several MCMC sampling methods give the factorized stationary
distribution given by equation\ (\ref{eq:stationary1}) and at the
same time are models of sexual reproduction that are as plausible as
those used in evolutionary computation or simplified models in
population genetics.

We suppose that each element  $x \in \X$ has a strictly positive weight $w(x)$. In context,
for brevity we will denote the weights $w(x_1), \ldots, w(x_n)$ as $w_1,\ldots,w_n$.
\subsection{Single tournament selection}
\label{sec:singletournamentselection}
We suppose that when a new genome $x_{n+1}$ is `born' and added to
the population, it competes to survive by having a tournament with
another randomly selected member of the population, $x_i$ say. The
probability that $x_{n+1}$ wins the tournament and ejects $x_i$ from
the population is $\frac{w_{n+1}}{w_{n+1} + w_i}$. This
probability is always well defined since $w$ is strictly positive.
An equally valid tournament winning probability is that $x_{n+1}$
wins with probability $\min\{1, \frac{w_{n+1}}{w_i}\}$.  These
two tournament winning probabilities are simply different
formulations of the Metropolis-Hastings acceptance rule; the proof
below establishes detailed balance for the first winning rule. The
algorithm for performing one generation of breeding, mutation, and
selection is:
\begin{enumerate}
\item Sample $ x_{n+1} \sim  P_\xi( \cdot \mid x_1, \ldots, x_n ) $
\item Sample $i$ randomly from $\{ 1, \ldots, n\}$
\item With probability $\frac{w_{n+1}}{w_i + w_{n+1}}$ replace $x_i$ with $x_{n+1}$ and discard $x_i$, otherwise discard $x_{n+1}$.
\end{enumerate}
Let $\bx$, $\bx^\prime$ be populations defined as
\begin{align*}
\bx &= x_1, \ldots , x_n\quad
\bx^\prime =  x_1, \ldots, x_{i-1},  x_{n+1}, x_{i+1}, \ldots, x_n
\end{align*}
\noindent Recall the measure $P_n(\bx)$ defined in
(\ref{eq:stationary1}). We now show that this measure satisfies
detailed balance. By exchangeability of $\xi$, we have:
\begin{equation}
P_\xi(\bx) P_\xi( x_{n+1} \mid \bx ) = P_\xi( x_1, \ldots , x_{n+1}) = P_\xi( \bx^\prime) P_\xi( x_i \mid \bx^\prime)
\end{equation}

\noindent Note that $x_{n+1}$ has tournament against $x_i$ with
probability $\frac{1}{n}$ and wins with probability
$\frac{w_{n+1}}{w_i + w_{n+1}}$, so:
\begin{align*}
P_n( \bx ) P( \bx \to \bx^\prime ) &=  P_n(\bx) \cdot P_\xi( x_{n+1} \mid \bx ) \cdot \frac{1}{n} \frac{w_{n+1}}{w_i + w_{n+1}} \\
&=\frac{1}{Z_n}  P_\xi( \bx ) w_1 \cdots w_n \cdot  P_\xi(x_{n+1} \mid\bx) \cdot \frac{1}{n} \frac{ w_{n+1}}{w_i + w_{n+1}}\\
&= \frac{1}{Z_n}P_\xi(x_1,\ldots, x_{n+1}) \cdot \frac{1}{n} \frac{w_1 \cdots w_{n+1}}{w_i + w_{n+1}} \\
&=  \frac{1}{Z_n} P_\xi(\bx^\prime) P_\xi(x_i \mid \bx^\prime) \cdot \frac{1}{n} \frac{w_1 \cdots w_{n+1}}{w_i + w_{n+1}} \\
&= \frac{1}{Z_n} P_\xi(\bx^\prime) w_1 \cdots w_{i-1} w_{n+1} w_{i+1} \cdots w_n \cdot P_\xi( x_i \mid \bx^\prime ) \cdot \frac{1}{n} \frac{w_i}{w_i + w_{n+1}}\\
& = P_n(\bx^\prime) P( \bx^\prime \to \bx).
\end{align*}

\subsection{Inverse fitness selection: limit of many tournaments}
\label{sec:inversefitnessselection} Suppose that many tournaments
are fought, and each time the loser of the previous tournament
fights another randomly chosen genome from the population. After
many tournaments, and at a stopping time, the current loser is
ejected. The current loser evolves according to a
irreducible aperiodic Markov chain, and the  limiting distribution
of ejection is the stationary distribution of that chain:
\begin{equation}
P(\text{ eject $i$ }) = \frac{\frac{1}{w_i} }{\frac{1}{w_1} + \cdots + \frac{1}{w_{n+1}} }.
\end{equation}
\noindent The algorithm for performing one generation of breeding and selection is then:
\begin{enumerate}
\item Sample $ x_{n+1} \sim  P_\xi( \cdot \mid x_1, \ldots, x_n ) $
\item Sample $i$ from  $\{1,\ldots,n+1\}$ with probabilities proportional to $\{ \frac{1}{w_1}, \ldots , \frac{1}{w_{n+1}}\}$
\item Discard $x_i$
\end{enumerate}
This process too satisfies detailed balance. With $\bx$ and
$\bx^\prime$ defined as above, note that:
\begin{align*}
P_n( \bx ) P( \bx \to \bx^\prime ) &=  P_n(\bx) \cdot P_\xi( x_{n+1} \mid \bx ) \cdot \frac{ \frac{1}{w_i} }{ \frac{1}{w_1} + \cdots + \frac{1}{w_{n+1}} }\\
&= \frac{1}{Z_n} P_\xi( \bx ) w_1 \cdots w_n \cdot  P_\xi(x_{n+1} \mid\bx) \cdot \frac{ \frac{1}{w_i} }{ \frac{1}{w_1} + \cdots + \frac{1}{w_{n+1}} }\\
&= \frac{1}{Z_n} P_\xi(x_1, \ldots, x_{n+1}) \frac{ w_1 \cdots w_{n+1}}{ w_{n+1} } \frac{ \frac{1}{w_i} }{ \frac{1}{w_1} + \cdots + \frac{1}{w_{n+1}} }\\
&= \frac{1}{Z_n} P_\xi(x_1, \ldots, x_{n+1}) \frac{ w_1 \cdots w_{n+1}}{ w_{i} } \frac{ \frac{1}{w_{n+1}} }{ \frac{1}{w_1} + \cdots + \frac{1}{w_{n+1}} }\\
&= P_n(\bx^\prime) \cdot P_\xi( x_i \mid \bx^\prime ) \cdot \frac{ \frac{1}{w_{n+1} }}{ \frac{1}{w_1} + \cdots + \frac{1}{w_{n+1}} }\\
& = P_n( \bx^\prime  ) P( \bx ^\prime\to \bx ).
\end{align*}
Note in passing that
\begin{equation}
\mathbf{E}\{ \mbox{weight of rejected genome} \mid w_1, \ldots , w_{n+1} \} = \frac{ 1}{ \frac{1}{w_1} + \cdots + \frac{1}{w_{n+1}} }.
\end{equation}
\noindent That is, the expected weight of the rejected genome is the harmonic mean of the
weights of the genomes in the current population, including the newly added genome $x_{n+1}$.
\paragraph{The case $n=1$ reduces to Metropolis-Hastings:}
with a population of size 1, we have proposal distribution
$P_\xi(x^\prime \mid x)  P_\xi(x) = P_\xi(x,x^\prime) = P_\xi(x \mid
x^\prime) P_\xi(x^\prime)$
and acceptance probability of $x^\prime$ given $x$ of $\frac{w(x^\prime)}{w(x) + w(x^\prime)}$, which gives a stationary distribution
\begin{equation*}
P_1(x) = \frac{P_\xi(x) w(x)}{\sum_x P_\xi(x) w(x) }.
\end{equation*}

\subsection{Exchangeable breeding of many offspring}
\label{sec:breedmany}

Another MCMC process which is also interpretable as an evolutionary algorithm breeds some arbitrary number $m$ of offspring to give a population of size $n+m$, and then from these selects $n$ genomes to form the next generation.  The algorithm for a single generation is as follows:

\begin{enumerate}
\item Pick a random number $m$ of offspring to breed, and a random number $t$ of tournaments to conduct. Both $m$ and $t$ should be independent of the current population $x_1, \ldots, x_n$
\item Breed $x_{n+1}, \ldots , x_{n+m}$ by sequential exchangeable sampling; that is, let $x_{n+i} \sim P_\xi(\cdot \mid x_1, \ldots, x_{n+i-1})$ for $1\le i \le m$
\item Assign `survival tickets' to each of $x_1, \ldots, x_n$;  the newly bred genomes $x_{n+1}, \ldots, x_{n+m}$ have as yet no survival tickets.
\item Repeat $t$ times:
\begin{enumerate}
\item Uniformly sample from the population a genome $x_i$ which currently has a survival ticket, and a genome $x_j$ which currently does not have a survival ticket.
\item Hold a tournament between $x_i$ and $x_j$; $x_i$ wins with probability $\frac{w(x_i)}{w(x_i) + w(x_j)}$.
\item The winner of the tournament gets the survival ticket; after the tournament, the winner has the ticket and the loser does not.
\end{enumerate}

\item After $t$ tournaments have taken place, the $n$ genomes currently holding survival tickets are selected to be the new
 population $x^\prime_1, \ldots, x^\prime_n$; the genomes that are not holding tickets are discarded.

\end{enumerate}

\noindent The arguments of section
\ref{sec:singletournamentselection} can readily be extended to show
that this algorithm has the same stationary distribution
(\ref{eq:stationary1}).
\subsection{The selection schemes and Moran models}
\label{sec:moran}
Recall that under the Dirichlet prior  our breeding process can be considered as Moran model with birth and mutation probabilities given in (\ref{mor}). Including the single tournament and  inverse fitness selections into the model, we end up with Moran model, where the transition probability that $n_i$becomes $n_i+1$ and $n_j$ becomes $n_j-1$ are, respectively,
\begin{equation}\label{morans}
\big({n_i\over n}(1-u)+u_i\big){n_j\over n}{w_i\over (w_i+w_j)},\quad \big({n_i\over n}(1-u)+u_i\big){w^{-1}_j\over \sum_{k}n_kw_k^{-1}+w^{-1}_i}.
\end{equation}
By the same arguments as in subsections \ref{sec:singletournamentselection} and \ref{sec:inversefitnessselection}, it is easy to verify that both transition matrices satisfy detailed balance equations with the stationary measure being (\ref{eq:stationary1}):
\begin{equation}\label{cnts2}
P_{n}(n_1,\ldots,n_K)={1\over Z_n}{n!\over n_1!\cdots n_K!}{(\alpha_1)_{n_1}\cdots (\alpha_K)_{n_K}\over (|\alpha|)_n}w_1^{n_1}\cdots w_K^{n_K}.
\end{equation}
Also, when adding the same single tournament or inverse fitness selections to the two allele model (\ref{tr2}), the detailed balance equations with (\ref{cnts2}) hold. The same holds for the decoupled Moran model (\ref{decoupled}): when adding single tournament or inverse fitness selections, the detailed balance equations with (\ref{cnts2}) hold. Of course, our selection schemes are far from the only possibilities for having (\ref{cnts2}) as stationary distribution. For example, when the mutation-drift Moran model (i.e. the probability of $j$-type to eject just being $n_j/n$, no selection) satisfies the detailed balance equation with  (\ref{cnts}) being the stationary distribution, then multiplying the transition probability that type $i$  is chosen to reproduce and $j$ to die simply by $w_i$ (as it is often done in the literature, e.g. \cite{BaakeBialowons,VoglClemente,Durrett}) then the detailed balance equation with (\ref{cnts2}) as stationary distribution holds. Thus in single tournament selection, the term ${w_i\over (w_i+w_j)}$ can be replaced  by $w_i$ and, again, detailed balance equations with (\ref{cnts2}) as stationary distribution hold (independently of drift-mutation probabilities, given that they satisfy detailed balance equation with (\ref{cnts})). To summarize: there are various different kinds of Moran models yielding the stationary distribution $P_n$ with $\Pi$ being the Dirichlet distribution

Typically in the evolution literature, for haploid models the weights $w_k$ are in form $(1+s_k)$, in two allele models simply $(1+s)$ and 1, where $s$ is either a positive or negative small number. In the large population limit, typically $s\cdot n\to \gamma$, where $\gamma$ is  constant. Hence, for large $n$, $(1+s)\approx (1+{\gamma\over n})\approx \exp[{\gamma\over n}]$ and this justifies our choice of weights  in (\ref{wn}). So, when in two allele model, $\phi(1)=0$ and $\phi(2)=\gamma>0$, thus $w_n(1)=1$ and $w_n(2)=e^{-\gamma/n}$, then the limit measure (\ref{l1dir}) has density
$${1\over Z_n} e^{-\gamma q_2}q_1^{\alpha_1-1}q_2^{\alpha_2-1}$$ that (with negative $\gamma$) is a classical and well-known limit, sometimes called Wright's formula \cite{BaakeBialowons} (see also (9) in \cite{VoglClemente}, 
Ch 5 in \cite{Ewens}, (7.28) in \cite{Durrett}, Ch 1 in \cite{Feng}, Ch 4 in \cite{Kingman}). We would like to stress   that we obtain our limit without any diffusion approximation and our main theorems \ref{thm1} and \ref{thm2} generalize this result in various ways.

\subsection{Invariance to fitness noise} \label{sec:noiseinvariance}

In real life, survival depends not just on fitness but also on luck.
Suppose that when each new `genome' is {`born'}, it combines its own
intrinsic fitness with an independent random amount of luck. It
keeps this amount of luck, unchanged, throughout its life. Does
individual luck at birth alter the stationary distribution?

To formalize the question, let $\psi_1,\psi_2,\ldots,\psi_n$  be discrete random variables that represent `luck'. We consider the $\psi_i$ as discrete random variables because the formal derivations become far simpler.  The
random variable $\psi_i$ can be interpreted as the individual
luck that multiplies the fitness of $i$-th individual. The log fitness
function of $i$-th individual is $\phi(x)+\psi_i$.
\begin{equation*}
P_n(\bx|\psi_1,\ldots,\psi_n) = \frac{1}{Z'_n} P_\xi( \bx ) \exp
\big[ -\sum_{i=1}^n \big(\phi(x_i) + \psi_i\big) \Big]
\end{equation*}
The joint probability
\begin{align*}
P_n(\bx;\psi_1,\ldots,\psi_n) &= \frac{1}{Z'_n} P_\xi(\bx) \exp
\big[ -\sum_{i=1}^n \big(\phi(x_i) +
\psi_i\big) \big] P(\psi_1,\ldots,\psi_n)\\
&=\frac{1}{Z'_n} P_\xi( \bx ) \exp \big[ -\sum_{i=1}^n
\phi(x_i)\big]\exp\big[-\sum_i \psi_i\big] P(\psi_1,\ldots,\psi_n).
\end{align*}
We see that the sum factorizes, and so $Z'_n=Z_n\cdot Z_n^{\psi}$,
where
$$Z_n=\sum_{\bx}P_\xi( \bx ) \exp \big[ -\sum_{i=1}^n
\phi(x_i)\big],\quad
Z^{\psi}_n=\sum_{\psi_1,\ldots,\psi_n}\exp\big[-\sum_i \psi_i\big]
\prod_i P(\psi_1,\ldots,\psi_n)$$
 Thus the joint measure factorizes
$$P_n(\bx;\psi_1,\ldots,\psi_n)=P_n(\bx)P_n^{\psi}(\psi_1,\ldots,\psi_n),$$
where
$$P_n^{\psi}(\psi_1,\ldots,\psi_n)={1\over
Z_n^{\psi}}P(\psi_1,\ldots,\psi_n)\exp[-\sum_i \psi_i]$$ is a
probability measure. Thus, after summing $\psi_1,\ldots,\psi_n$ out,
we end with  $P_n(\bx)$:
$$\sum_{\psi_1,\ldots,\psi_n}P_n(\bx;\psi_1,\ldots,\psi_n)=P_n(\bx).$$
It follows that  the stationary
distribution is unchanged by multiplicative noise that is
independent of the genomes, for all population sizes. Also,
considering $\psi_1,\ldots,\psi_n$ as the prior, we see that
posterior measure is  $P_n(\psi_1,\ldots,\psi_n)$ that is
independent of $\bx$.

\subsection{Fitness as likelihood: a connection with non-parametric Bayesian MCMC}
\label{sec:BayesianMCMC}

In Dirichlet Process mixture models, as described for example in \cite{neal2000markov,teh2010hierarchical}, items of data $d_1, \ldots, d_n$ are given, together with a likelihood function $l(x,d) = P(d\mid x)$, where $x \in \mathcal{X}$ is a discrete latent variable. The prior distribution over latent variables is given by the exchangeable process $\xi$, each item of data is associated with its corresponding latent variable, and the aim of MCMC fitting is to sample latent variables $x_1, \ldots, x_n$ from the distribution
\begin{equation}
\label{eq:likelihood}
P_n(x_1, \ldots, x_n) = \frac{1}{Z_n} P_\xi(x_1, \ldots, x_n) P(d_1 \mid x_1) \cdots P(d_n \mid x_n)
\end{equation}
where $Z_n$ is an appropriate normalizing constant. If we write
\begin{equation*}
{w_j( x_i) := P(d_j \mid x_i)}
\end{equation*}
then this distribution (\ref{eq:likelihood}) is produced by the MCMC
algorithm of section \ref{sec:singletournamentselection} with the
slight modification that the tournament between the new `genome'
$x_{n+1}$ and the randomly selected $x_j$ is a victory for $x_{n+1}$
with probability $\frac{w_j(x_{n+1})}{w_j(x_j) + w_j(x_{n+1})}$.

One might construct an evolutionary ``Just-So'' story as follows. On
a rock in the ocean, there are $n$ niches, in each of which one
member of the species $\xi$ can live; the fitness of $x$  living in
niche $j$ is $w_j(x)$.   Evolution occurs when a new individual,
bred by `n-way recombination' of all $n$ parents, then challenges
the occupant of a randomly chosen niche by fighting a tournament,
after which the victor survives and takes over the niche. This gives
a (fanciful) evolutionary interpretation to Bayesian latent variable
models with exchangeable priors.

\subsection{Designing a reversible evolutionary model}

Our larger aim is to develop a model of evolution that is
sufficiently realistic to capture some of the computational power of
natural evolution, but which is also simple and tractable for
analysis.  To ensure that our model satisfies detailed balance and
has a stationary distribution that factorises into a breeding and a
selection term, we have made the following simplifying assumptions
in addition to the usual simplifications of population genetics or
genetic algorithms:

\begin{description}
\item[Overlapping populations]: We believe that overlapping populations are necessary for reversibility. If full replacement of the population is enforced at each generation, there can be do guarantee that the population at time $t$ could be easily bred from the population at time $t+1$.  In MCMC, state changes typically occur through proposing changes that may or may not be accepted; in an overlapping generations model, if a proposed change is not accepted, we continue with the same population as before.

\item[$n$-way recombination]: in the product of Dirichlet processes breeding system,  each new genome is bred from $n$ parents rather than from two parents selected from the population. We conjecture that this is necessary for exact reversibility because, with long genomes and many mutations,  if a child is bred from two parents, then the child will be more similar to each of its parents than to other individuals in the population, so that `triples' of two parents and one child will be identifiable even in the stationary distribution. This breaks reversibility since the direction of time can be determined observing evolution in the stationary distribution.

\item[Mutation as sampling]: We consider mutation as sampling from a base distribution of possible alleles. This model of mutation is not as general as those found in biology, where mutation probabilities are not symmetric or reversible.

\item[Fitness as lifetime]: All members of the population `breed' at the same rate, and differences in fitness affect only the expected lifetime of an individual. This clearly differs from many types of natural selection, but it is also well known that many organisms continue to produce offspring throughout their lives, so that their total reproductive success depends on their lifetime, as well as on other factors.
\end{description}

It is beyond the scope of this article to argue further whether our model successfully abstracts some essential computational aspects of evolution with sexual reproduction: we present it merely as an abstraction of sexual evolution which is significantly more tractable to analyse than other apparently simple models.

\section{The measure $P_n$}
\label{sec:measurepn}
The measure $P_n$ in (\ref{eq:stationary1})  is our main object of
interest. In this section we show that the marginal distributions
$P_n$ over the set of genotypes converge as the population size $n
\rightarrow \infty$; in the next section we characterize these
limits.

Since $\xi$ is exchangeable, by de Finetti's theorem there exists a
prior measure $\pi$ on {the set of all probability measures on $\X$
(simplex)} ${\cal P}$ such that for every $\bx\in \X^n$, the measure
$P_{\xi}(\bx)$ allows the representation (\ref{fi}). Hence we can
write $P_n$ as follows
\begin{equation}\label{measure2}
P_n(\bx)={1\over Z_n}\prod_{i=1}^n w(x_i)\int_{\P} \prod_{i=1}^n
 q(x_i)\pi(d q)={1\over Z_n}\int_{\P} \prod_{i=1}^n
 \big(q(x_i)w(x_i)\big)\pi(d q),
\end{equation}
where $Z_n$
is the normalizing constant. In order to analyze the measure, it is
convenient to rewrite it as follows. First, let us introduce some
notation
$$\langle q,w \rangle:=\sum_{k=1}^K q(k)w(k),\quad r_q(k):={w(k)q(k)\over \langle q,w \rangle},\quad k=1,\ldots,K.$$
Note that since $w(k)>0$ for all $k$, $r_q$ is correctly defined for
every $q \in \mathcal{P}$. Thus $\langle q,w \rangle$ is the
expected weight (under  $q$-measure)
 and $r_q$ is a probability measure on $\P$.  Now
\begin{equation}\label{eq:pn}
P_n(\bx)={1\over Z_n} \int_{\P} \prod_{i=1}^n
 \big(q(x_i)w(x_i)\big)\pi(d q)={1\over Z_n}\int_{\P} \langle q,w \rangle^n \prod_{i=1}^n r_q(x_i) \pi
 (dq).\end{equation}
Since $r_q$ is a probability measure, it holds that $\sum_{\bx}
\prod_{i=1}^n r_q(x_i)=1$ and so the normalization constant for
$P_n$ is
$$Z_n=\sum_{\bx\in \X^n}\int_{\P} \langle q,w \rangle^n \prod_{i=1}^n r_q(x_i) \pi
(dq)=\int_{\P} \langle q,w \rangle^n  \pi (dq).$$ Finally note that
(\ref{eq:pn}) can be rewritten more neatly by defining the measure
\begin{equation}\label{eq:pi}
d\bar{\pi}_n:={\langle q,w \rangle^n\over Z_n}d\pi.\end{equation}
With $\bp$, we have
\begin{equation}\label{pibn}
P_n(\bx)=\int_{\P}\prod_{i=1}^n r_q(x_i) \bp (dq)\quad \forall \quad \bx\in \X^n.
\end{equation}
From(\ref{pibn}), it is easy to find all marginal distributions,
namely for any $m=1,\ldots,n$
\begin{equation}\label{marginals}
P_n(x_1,\ldots,x_m)=\int_{\P}\prod_{i=1}^m r_q(x_i) \bp (dq)={1\over
Z_n}\int_{\P} \langle q,w \rangle^{n-m} \prod_{i=1}^m q(x_i)w(x_i)
\pi (dq).
\end{equation}
In particular, when $(X_1,\ldots,X_n)\sim P_n$, then
\begin{align*}
P(X_i=k)&=\int_{\P}r_q(k) \bp (dq)={1\over Z_n}\int \langle q,w
\rangle^{n-1} w(k) q(k) \pi (dq)\\
P(X_i=k,X_j=l)&=\int_{\P}r_q(k)r_q(l) \bp (dq)={1\over Z_n}\int
\langle q,w \rangle^{n-2}w(k)q(k)w(l)q(l) \pi (dq)
\end{align*}
and so on. It is important to observe that $P_n(x_1,\ldots,x_m)$
depends on $n$.
\paragraph{The limit process.} We have defined for every $n$ the measure (\ref{pibn})
that describes the genotype distribution of a $n$-element
population. Now the natural question is: do these measures converge
(in some sense) if the population size $n$ grows? First we have to
define the sense of convergence. Since every measure $P_n$ is
defined on different domain ($\X^n$), we cannot speak about standard
(weak) convergence of measures. Instead, we ask about the existence
of a limiting stochastic process. To explain the sense of
convergence, consider that we have defined a triangular array of
random variables:
\begin{align*}
&X_{1,1}\sim P_1\\
&(X_{2,1},X_{2,2})\sim P_2\\
&(X_{3,1},X_{3,2},X_{3,3})\sim P_3\\
&\cdots \\
&(X_{n,1},X_{n,2},\ldots,X_{n,n})\sim P_n\\
&\cdots
\end{align*}
We also know that the joint distribution of the first $m$ variables
in every row depends on $n$.  Therefore we ask: is there a
stochastic process $X_1,X_2,\ldots$ so that for every $m$ the
following convergence holds
\begin{equation}\label{convm}(X_{1,n},\ldots X_{m,n})\Rightarrow
(X_{1},\ldots X_{m})?\end{equation}
According to Kolmogorov's existence theorem, the existence of a
stochastic process is equivalent to the existence of (finite
dimensional) measures $P^*_m$ on set $\X^m$, $m=1,2,\ldots$  that
satisfy  the following consistency conditions: for every $m$ and for
every $(x_1,\ldots,x_m)\in \X^m$, it holds that
$$\sum_{x_{m+1}}P_{m+1}^*(x_1,\ldots,x_m,x_{m+1})=P_{m}^*(x_1,\ldots,x_m).$$
If we also want (\ref{convm}) to be true, then for every $m$ and for
every $(x_1,\ldots,x_m)\in \X^m$ the following convergences must
hold:
\begin{equation}\label{fin-dim}
P_n(x_1,\ldots,x_m)\to P^*(x_1,\ldots,x_m),\quad \forall m,\quad
\forall (x_1,\ldots,x_m)\in \X^m.\end{equation} We now present a
general lemma that guarantees the convergence (\ref{fin-dim}). To
achieve the full generality, we let $w$ also depend on $n$. Thus, we
have weights $w_n$, and we  define the measures $r_{q,n}$ as follows
$$
r_{q,n}(k):={w_n(k)q(k)\over \langle q,w_n \rangle} \quad \forall
k\in \X.$$
 We start with the following observation, proven in appendix.
\begin{claim}\label{claim} If  $w_n(k)\to w(ik)$ $\forall k\in \X$, and $r_{q,n}$ and $r_q$ are defined with respect of $w_n$ and $w$, respectively,
then the following uniform convergence holds.
\begin{equation}\label{unif}
\sup_{q\in \P}|r_{q,n}(k)-r_q(k)|\to 0.
\end{equation}
\end{claim}
 In the following lemma,  $\bp$ {are}
 arbitrary probability measures on $\P$, not necessarily as in
 (\ref{eq:pi}). The measure $\bp$  define $P_n$ as in (\ref{pibn}).
\begin{lemma}\label{lemma1} Let $w_n(k)\to w(k)$ for every $k\in \X$. If there exists an probability measure $\bar{\pi}$ such
 that $\bp\Rightarrow \bar{\pi}$, then for every $m$ there exists a probability measure
 $P_m^*$  on $\X^m$ so that (\ref{fin-dim})
 holds. Moreover, for every $(x_1,\ldots,x_m)\in \X^m$
$$P^*_m(x_1,\ldots, x_m)=\int_{\P}\prod_{i=1}^m
r_{q}(x_i)\bar{\pi}(dq),\quad {\rm where}\quad r_q(k)={w(k)q(k)\over
\langle w, q \rangle},\quad k=1,\ldots,K
$$ and the measures $P^*_m$, $m=1,2,\ldots$ satisfy consistency
conditions.\end{lemma}
\begin{proof}
For every $x_1,\ldots, x_m$ from (\ref{unif}), it follows that
\begin{equation}\label{unif2}
\sup_{q\in \P}|\prod_{i=1}^m r_{q,n}(x_i)-\prod_{i=1}^m r_q(x_i)|\to
0.
\end{equation}
Since the functions
$$q\mapsto \prod_{i=1}^m r_{q,n}(x_i),\quad q\mapsto \prod_{i=1}^m r_q(x_i)$$
are bounded (by 1) continuous functions, from uniform convergence,
it holds
$$P_n(x_1,\ldots,x_m)=\int_{\P}\prod_{i=1}^m r_{q,n}(x_i)\bp(dq)\to \int_{\P}\prod_{i=1}^m
r_{q}(x_i)\bar{\pi}(dq){=:}P^*_m(x_1,\ldots, x_m).$$
Clearly $P_m^*$ are probability measures. The consistency condition
trivially holds, because
$$\sum_{m+1}P^*_{m+1}(x_1,\ldots,x_{m+1})=\int_{\P}\sum_{x_{m+1}}\prod_{i=1}^{m+1}
r_{q}(x_i)\bar{\pi}(dq)=\int_{\P}\prod_{i=1}^m
r_{q}(x_i)\bar{\pi}(dq)=P^*_m(x_1,\ldots, x_m).$$\end{proof}
\subsection{Frequencies: the measure $Q_n$}\label{sec:kuu}
We now consider how to express the limit measure $P^*$ in terms of a
limiting measure on the simplex $\mathcal{P}$. Recall $n_k(\bx)$
defined in (\ref{fr}) and let
\begin{equation*}
\mathbf{n}(\bx)  := (n_1, \ldots, n_K ), \quad \text{where $n_k=
n_k(\bx)$, for $k=1,\ldots,K$}
\end{equation*}
Since $\xi$ is exchangeable,  the probability $P_\xi(\bx)$ depends
on the counts $\mathbf{n}(\bx)$ only,  so
\begin{equation*}\label{g}
P_{\xi}(\bx)=:g(n_1,\cdots,n_K).
\end{equation*}
We may now write:
\begin{equation}\label{definetti}
g(n_1,\dots,n_K)=P_{\xi}(\bx)=\int_{\P}\prod_{i=1}^n
q(x_i)\pi(dq)=\int_{\P}\prod_{k=1}^K q(k)^{n_k}\pi(dq).
\end{equation}
In what follows, let us denote
$$\mathbb{N}_n:=\{(n_1,\ldots,n_K): \sum_k n_k=n\}.$$
Observe that
$$\sum_{(n_1,\ldots,n_k)\in \mathbb{N}_n}{n!\over n_1!\cdots n_K!}
g(n_1,\cdots,n_K)=1.$$ \noindent Therefore, the  measure $P_n$ can be defined
on the set $\mathbb{N}_n$ as follows:
\begin{align}\label{measuren} P_n({\bf n})&
:={1\over Z_n}\frac{n!}{n_1!\cdots
n_K!}g(n_1,\ldots,n_K)\prod_{k=1}^K \big( w_n(k)\big)^{n_k}={n!\over n_1!
\cdots n_K!} \int_{\P} \prod_{k=1}^K \big(r_{q,n}(k)\big)^{n_k} \bp
(dq).
\end{align}
Considering the frequencies instead of counts, we can define the
corresponding measure on the simplex $\P$. Let us denote that
measure as $Q_n$, so that with ${\bf n}/n:=(n_1/n,\ldots n_k/n)$
\begin{equation}\label{Qn}
Q_n\biggl({{\bf n}\over n}\biggr):=P_n({\bf n}),\quad \forall {\bf n}\in
\mathbb{N}_n.\end{equation} Thus $Q_n$ is a discrete measure
$$Q_n=\sum_{{\bf n}\in \mathbb{N}_n}
P_n({\bf n})\delta_{{\bf n}\over n}.$$ The advantage of $Q_n$ over
the measure $P_n$ on $\mathbb{N}_n$ is that for any $n$, $Q_n$ is
defined on the same domain $\P$, and so one can speak about the weak
convergence of $Q_n$. Essentially, obviously, the measure $P_n$ on
$\X^n$, the measure $P_n$ on $\mathbb{N}_n$ and $Q_n$ on $\P$ are
all the same, just the domains are different.

Since the measures $Q_n$ are defined on the same space (simplex), it
is now natural to ask, whether there exists a probability measure
$Q^*$ so that $Q_n\Rightarrow Q^*$? It turns out the if the
assumption of Lemma \ref{lemma1} holds, i.e.  $\pi_n\Rightarrow
{\bar \pi}$ and $w_n\to w$ (pointwise), then the limit measure is
actually ${\bar \pi}r^{-1}$, where
$$r: \P \mapsto \P,\quad r(q)=r_q$$
and  $r$ is defined with respect to limit weight function $w$. Thus
for a measurable $E\subset \P$,
$${\bar \pi}r^{-1}(E)={\bar \pi}\big(r^{-1}(E)\big).$$
 For example, if ${\bar \pi}=\delta_{q^*}$ (the measure is
concentrated on one point), then
$${\bar \pi}r^{-1}=\delta_{r(q^*)},$$ because
$$\delta_{q^*}r^{-1}(E)=1\quad \Leftrightarrow \quad q^*\in
r^{-1}(E)\quad \Leftrightarrow \quad r(q^*)\in E.$$ The following
lemma  is the counterpart of Lemma \ref{lemma1}. Again, $\bp$ is an
arbitrary sequence of probability measures on $\P$, $P_n$ are
defined via $\bp$ by (\ref{measuren}) and $Q_n$ via $P_n$ as in
(\ref{Qn}).
\begin{lemma}\label{lemma2} Let $w_n(k)\to w(k)$ for every $k\in \X$.   If there exists  a probability measure $\bar{\pi}$ such
 that $\bp\Rightarrow \bar{\pi}$, then $Q_n\Rightarrow {\bar
 \pi}r^{-1}$.\end{lemma}
\begin{proof} Let $f: {\cal P}\to \mathbb{R}$ be a $(K$-variable) bounded
continuous function. By definition of the weak convergence, it
suffices to show that
\begin{equation}\label{int}
\int f(q) Q_n(dq)\to \int f(q) {\bar \pi}r^{-1}(dq)=\int f(r_q){\bar
\pi}(dq),\end{equation}
where the last equality holds by the change of variable formula.
Note that
\begin{align*}
\int f(q)Q_n(dq)&=\sum_{{\bf n}\in \mathbb{N}_n} f\Big({{\bf n}\over
n}\Big)P_n({\bf n})=\sum_{{\bf n}\in \mathbb{N}_n} f\Big({{\bf
n}\over n}\Big){n!\over n_1! \cdots n_K!} \int \prod_{k=1}^K
\big(r_{q,n}(k)\big)^{n_k} \bp (dq)\\
&=\int \Big(\sum_{{\bf n}\in \mathbb{N}_n} f\Big({{\bf n}\over
n}\Big) {n!\over n_1! \cdots n_K!} \prod_{k=1}^K
\big(r_{q,n}(k)\big)^{n_k} \Big)
\bp(dq)\\
&=\int f_n(r_{q,n}(1),\ldots,r_{q,n}(K)) \bp (dq)=\int f_n(r_{q,n})
\bp (dq) ,\end{align*} where
$$f_n(r_{q,n}):=:f_n(r_{q,n}(1),\ldots,r_{q,n}(K)):=\sum_{(n_1,\ldots,n_K)\in \mathbb{N}_n:} f\Big({n_1\over n},\ldots,{n_1\over n}\Big){n!\over n_1!
\cdots n_K!} \prod_{k=1}^K \big(r_{q,n}(k)\big)^{n_k}$$ is the
Bernstein polynomial  evaluated at
$r_{q,n}=(r_{q,n}(1),\ldots,r_{q,n}(K))$. It is easy to see and well
known that for any vector $r\in \P$, $f_n(r)\to f(r)$,  moreover, the
convergence is uniform over $\P$:
\begin{equation}\label{bern}
\sup_{r\in \P}\Big|f_n(r)-f(r)\Big|\to 0.\end{equation} Since for
every $q$, $r_{q,n}$ is a probability vector, then
$${n!\over n_1! \cdots n_K!} \prod_{k=1}^K
\big(r_{q,n}(k)\big)^{n_k}\leq 1,$$ and since $f$ is bounded, we see
that for every $n$, $q\mapsto f_n(r_{q,n})=:b_n(q)$ is a bounded
continuous function. Also the function $q\mapsto f(r_{q})=:b(q)$ is
a bounded continuous function. Then
\begin{align*}
\sup_q|b_n(q)-b(q)|&\leq \sup_q|f_n(r_{q,n})-f(r_{q,n})|+
\sup_q|f(r_{q,n})-f(r_{q})|.
\end{align*}
By (\ref{unif}), $\sup_q|r_{q,n}(k)-r_q(k)|\to 0$ for every $k\in
\X$. Then also  $\sup_q\|r_{q,n}-r_q\|\to 0$. A continuous function
on compact space is uniformly continuous, so
$$\sup_q|f(r_{q,n})-f(r_{q})| \to 0.$$ By (\ref{bern}),
$$\sup_q|f_n(r_{q,n})-f(r_{q,n})|\leq \sup_r|f_n(r)-f(r)|\to 0.$$
Therefore, we have shown that $\sup_q|b_n(q)-b(q)|$ implying that
$$\int f(q)Q_n(dq)=\int b_n(q)\bp(dq)\to \int b(q){\bar \pi}(dq)=\int f(r_{q}){\bar \pi}(dq).$$
\end{proof}
\section{$P^*$ and $Q^*$ in the large population limit}
\label{sec:largepopulationlimit} In what follows, let us rewrite
fitnesses in terms of $\phi(k) := - \ln( w(k))$, so that for any
genome $k\in \X$, $w(k) = \exp(-\phi(k))$. Moreover, in order to
increase the influence of prior, we let weights  $w_n$ depend on $n$
{as defined in (\ref{wn}), i.e. $w(k) =
\exp(-\phi(k)/n^{\lambda})$}.  Clearly, for every $k$, $w_n(k)\to
w(k)$ and $\lambda$ controls the speed of that convergence. When
$\lambda>0$, then $w(k)=1$ implying that in this case the mapping
$r$ is identity, i.e. for every $q$, $r_q=q$.

Let us return to our original $\bp$, defined as in (\ref{eq:pi})
with $w_n$:
\begin{equation}\label{pinn}
\bar{\pi}_n(E)=\int_E {\langle q,w_n \rangle^n\over Z_n}\pi (dq),\quad
Z_n=\int \langle q,w \rangle^n \pi(dq).\end{equation} In this section
we consider the case where the prior measure $\pi$ is independent of
$n$, and the support of $\pi$ is the whole simplex ${\cal P}$.
When  $w(1)>w(2)$, the function $q\mapsto \langle q,w \rangle$ has
unique maximizer $q^*:=(1,0,\ldots,0)$. The following theorem states that under $w(1)>w(2)$
that phase transition occurs: when $0\leq \lambda <1$, then
 $\bar{\pi}_n\Rightarrow \delta_{q^*}$ and $P_n(x_1,\ldots,x_m)\to
P^*(x_1,\ldots,x_m),$ where
$$P^*(x_1,\ldots, x_m)=\prod_{i=1}^m r_{q^*}(x_i)=\prod_{i=1}^m
q^*(x_i)=\left\{
            \begin{array}{ll}
              1, & \hbox{if for every $i$, $x_i=1$ ;} \\
              0, & \hbox{else.}
            \end{array}
          \right.,
$$
 because in both cases (i.e. $\lambda=0$ and $\lambda\in (0,1)$), it holds that $r_{q^*}=q^*$.
 Thus the limit process $X_1,X_2,\ldots$ has
only one realization: $1,1,\ldots$.  When $\lambda=1$, then
$\bar{\pi}_n$ as well $Q_n$ converges to a nondegenerate
distribution, when $\lambda>1$, then $\bar{\pi}_n\Rightarrow \pi$,
$Q_n\Rightarrow \pi$ and the measure $P^*$ is the law of birth
process $\xi$.
\begin{theorem}\label{thm1} Let the fitness function be defined as in (\ref{wn}) and assume that the support of the prior $\pi$  is ${\cal P}$. Then the following convergences hold:
\begin{description}
  \item[1)] If $\lambda \in [0,1)$ and $\phi(1)<\phi(2)$, then $\bar{\pi}_n\Rightarrow \delta_{q^*}$, $Q_n\Rightarrow \delta_{q^*}$
  and (\ref{fin-dim}) holds with
  $$P^*(x_1,\ldots,x_m)=\prod_{i=1}^mq^*(x_i),\text{   where  }
  q^*=(1,0,\ldots,0).$$
  \item[2)] If $\lambda=1$, then $\bar{\pi}_n\Rightarrow {\bar \pi}$, $Q_n\Rightarrow {\bar \pi}$ and (\ref{fin-dim}) holds with $$P^*(x_1,\ldots,x_m)=\int \prod_{i=1}^mq(x_i)\p(dq),$$ where for every $E\subset {\cal P}$,
  $${\bar \pi}(E)={1\over Z} \int_E \exp[-\langle \phi,q\rangle] \pi (dq),\quad Z=\int \exp[-\langle \phi,q\rangle] \pi (dq).$$
  \item[3)]  If $\lambda>1$, then $\bar{\pi}_n\Rightarrow {\pi}$, $Q_n\Rightarrow {\pi}$ and (\ref{fin-dim}) holds with
  $$P^*(x_1,\ldots,x_m)=P_{\xi}(x_1,\ldots,x_m).$$
\end{description}
\end{theorem}
\subsection{Proof of Theorem \ref{thm1}}
Before proving the theorem, let us state a very useful preliminary
result. Recall that simplex $\P$ is a compact set.
Let $f_n,f: {\cal P}\to \mathbb{R}^+$ be {continuous, hence bounded}
measurable functions so that $f_n\to f$ uniformly and let $m_n\to
\infty$ be an increasing sequence. We are given a measure $\pi$ on
${\cal P}$ {having full support (i.e. the support of
$\pi$ is ${\cal P}$)}, and we are interested in the asymptotic
behavior of the measure $\nu_n$, where
$$\nu_n(E):=\int_E h_n(q) \pi(dq),\quad   h_n(q):={f_n^{m_n}(q)\over \int f_n^{m_n}(q) \pi(dq)}=\Big({f_n(q)\over
\|f_n\|_{m_n}}\Big)^{m_n}.$$
Here  {$\|f\|_m$ stands for $L_m$-norm and} we assume
that $\int  f_n^{m_n} d\pi<\infty$ for every $n$. If $\pi$ is a
finite measure, then the conditions automatically holds due to the
boundedness of $f_n$.
 In what follows, let
{\begin{equation}\label{S} {\cal P}^*:=\{q\in {\cal
P}: f(q)=\|f\|_{\infty}\},\quad {\cal P}_{\delta}^*:=\{q\in {\cal
P}: f(q)> \|f\|_{\infty}-\delta\}.\end{equation}} Here
$\|f\|_{\infty}$ is the essential supremum of $f$ with respect to
the $\pi$-measure. If $f$ is continuous, then $\|f\|_{\infty}=\sup_q
f(q)$ {(because $\pi$ has full support)}. The proof
of the following Proposition \ref{point2} is given in appendix.
\begin{proposition}\label{point2} Let  $f_n\to f$ uniformly and {let $\pi$ be a finite measure on
$\P$ {having full support}. Then for every
$\delta>0$, $\nu_n\big({\cal P}_{\delta}^*\big)\to 1$.} If ${\cal
P}^*=\{q^*\}$, then $\nu_n\Rightarrow \delta_{q^*}.$
\end{proposition}
Besides Proposition \ref{point2}, the proof of Theorem \ref{thm1} is
based on the following well-known observation: when $m\to \infty$,
then
\begin{equation}\label{exp}
\sup_{q\in {\cal P}}\Big| \big\langle \exp[-{\phi\over m}], q \big\rangle^m - \exp[-\langle \phi, q
\rangle]\Big|\to 0.
\end{equation}
\begin{proof} {\bf (Theorem \ref{thm1})}
\begin{description}
\item[1)] For $\lambda=0$, take $f_n(q)=f(q)=\langle w,q \rangle$, {$m_n=n$. By $w(1)>w(2)\geq \cdots \geq w(K)$, we have ${\cal P}^*=\{q^*\}$
}, from Proposition \ref{point2}, it follows that $\bp\Rightarrow
\delta_{q^*}$. Since for any weight $w$, $r_{q^*}(k)=q^*(k)$,
from Lemma \ref{lemma1}, it follows that $$P^*(x_1,\ldots,
x_m)=\prod_{i=1}^m r_{q^*}(x_i)=\prod_{i=1}^m q^*(x_i).$$ Since
$\bp r^{-1}(q^*)=q^*$, from Lemma \ref{lemma2}, it follows that
$Q_n\Rightarrow \delta_{q^*}$. So, for $\lambda=0$, the
statement is proven and we  now consider the case $\lambda\in
(0,1)$. Let
$$f_n(q):=\langle \exp[-{\phi\over n^{\lambda}}], q
\rangle^{n^{\lambda}},\quad f(q):=\exp[-\langle \phi, q \rangle
].$$ By (\ref{exp}), $\|f_n-f\|_{\infty}\to 0$. Since $\lambda\in (0,1)$,
take $m_n=n^{1-\lambda}$. Then
$$h_n(q):={\langle \exp[-{\phi\over n^{\lambda}}], q
\rangle^{n}\over Z_n}$$  is the density of $\bp$ with respect to
$\pi$. Since $f$ is continuous, the set ${\cal P}^*$ in (\ref{S}) is
$${\cal P}^*= \arg\max_{q\in {\cal P}} f(q)=\arg\min_{q\in {\cal P}}\langle \phi, q \rangle=\{(1,0,\ldots,0)\}=\{q^*\}.$$
 By Proposition \ref{point2}, $\bp\Rightarrow \delta_{q^*}$. As
 in the case of $\lambda=0$, it follows that $Q_n\Rightarrow
\delta_{q^*}$ and $P^*(x_1,\ldots,x_m)=1$ if and only if
$x_1=\cdots=x_m=1$.
  \item[2)] Since for any $q$ and any $n$, it
holds $$\langle \exp[-{\phi\over n}], q \rangle^n\leq
e^{-\phi(1)}\leq 1,$$ we obtain from (\ref{exp}) and bounded
 convergence that for any
measurable $E$
\begin{equation}\label{Ekoo}
\int_E \langle \exp[-{\phi\over n}], q \rangle^n \pi (dq)\to \int_E
\exp[-\langle \phi, q \rangle]\pi(dq).\end{equation}
Recall
$$\p(E)={1\over Z}{\int_E}\exp[-\langle \phi, q
\rangle]\pi(dq),\quad {\rm where}\quad Z=\int \exp[-\langle \phi, q
\rangle]\pi(dq).$$ Therefore, from (\ref{Ekoo}), it follows that
when $\lambda=1$, we have {for every measurable $E$}
$$\bp(E)\to {\bar\pi}(E),$$
meaning that  $\bp\Rightarrow \p$ (even in a stronger sense).
Since $r_q=q$, from Lemma \ref{lemma1}, it follows that the
limits of $P_n(x_1,\ldots,x_m)$ are
$$P^*(x_1,\ldots,x_m)=\int \prod_{i=1}^mq(x_i)\p(dq)={1\over Z}\int \prod_{i=1}^mq(x_i)\exp[-\langle \phi, q \rangle]\pi(dq).$$
Since $r$ is identity function, by Lemma \ref{lemma2} the limit
measure of frequencies is $\p$, i.e. $Q_n\Rightarrow \p$.
  \item[3)] Since for any $q$,
$$\langle \exp[-{\phi\over n^{\lambda}}], q \rangle^n\to 1$$
 by dominated convergence, again, for any measurable $E$
\begin{equation*}\label{Ekoo2}
\int_E \langle \exp[-{\phi\over n}], q \rangle^n \pi (dq)\to
\pi(E).\end{equation*} Therefore $\bp\Rightarrow \pi$. By Lemma \ref{lemma1},  the limits of
$P_n(x_1,\ldots,x_m)$ are
$$P^*(x_1,\ldots,x_m)=\int \prod_{i=1}^mq(x_i)\pi(dq),$$
so that the limit process is $\xi$. The convergence
$Q_n\Rightarrow \pi$ follows from Lemma \ref{lemma2}.
\end{description}
\end{proof}

We have seen that the critical case $\lambda=1$ is the only case
where the prior and fitnesses both determine the limit measure. In
this case, the limit process $X_1,X_2,\ldots$, governed by $P^*$ has
marginals
\begin{align*}
P(X_i=k)&=P(X_1=k)=\int q(k) \bar{\pi}(dq)={1\over Z}\int e^{-\langle \phi,q
\rangle}q(k)\pi(dq),\\
P(X_i=k,X_j=l)&=P(X_1=k,X_2=l)= {1\over Z}\int e^{-\langle \phi,q
\rangle}q(k)q(l)\pi(dq).\end{align*} It is also interesting to point out that in the critical case $\lambda=1$, the measure $\bar{\pi}$ satisfies
$$\bar{\pi}=\arg\min_{\pi'\in E}D(\pi\|\pi'),$$
where $E$ is a set of probability measures on $\P$, namely
$E:=\{\pi': \int \langle \phi,q \rangle \pi'(dq)\geq c\}$,
{$D$ stands for Kullback-Leibler divergence and}
$c>0$ is a constant.
\section{Dirichlet prior}
\label{sec:dirichletprior}
Also in the current section we consider the weights $w_n(k)$ as in
(\ref{wn}), where $\lambda \in [0,1]$. We already know that in the
case of constant priors, the case $\lambda<1$ means that the
fitnesses will prevail over the prior. Therefore, it is meaningful
to consider non-constant priors so that the influence of prior
increases with suitable rate. This motivates us to consider
 Dirichlet priors (\ref{dir}), i.e.
$\pi_n={\rm Dir}(n^{1-\lambda}\alpha_1,\ldots
n^{1-\lambda}\alpha_K),$  where $\a:=(\alpha_1,\ldots,\a_K)$,
$\alpha_k>0$ and $|\a|:=\sum_k \a_k$. The constant $\sum_i
n^{1-\lambda}\a_k=|\a|n^{1-\lambda}$ is the so called {\it
concentration} or {\it precision} parameter, the bigger that
parameter, the more prior is concentrated over its expectation
$(\a_1/|\a|,\ldots,\a_K/|\a|)$. Increasing the concentration
parameter increases the influence of prior, and now it is clear the
the smaller is $\lambda$, the bigger must be the prior influence.
This justifies the choice of $n^{1-\lambda}$. The case $\lambda=1$
corresponds to already studied case of constant priors, therefore we
now consider the case $\lambda\in [0,1)$. The following theorem
shows that the phase transition occurs again.
\begin{theorem}\label{thm2} Let the fitness function be defined be defined as in
(\ref{wn}) and the prior $\pi_n$ as in (\ref{dir}). Let $\bp$ be
defined as in (\ref{pinn}) with $\pi_n$ instead of $\pi$.  Then the
following convergences hold:
\begin{description}
  \item[1)] If $\lambda=0$, then  $\bar{\pi}_n\Rightarrow
  \delta_{q^*}$, where $q^*$ is the unique maximizer of the
  following function
\begin{equation}\label{measure1}
  \ln \langle e^{-\phi} ,q \rangle + \sum_{k=1}^K\alpha_k\ln
q(k).\end{equation} Then $Q_n\Rightarrow \delta_{r^*}$, where
$r^*=r_{q^*}$, so that $r^*(k)\propto q^*(k)w(k)$, and
(\ref{fin-dim}) holds with
$$P^*(x_1,\ldots,x_m)=\prod_{i=1}^mr^*(x_i).$$
  \item[2)] If $\lambda\in (0,1)$, then  $\bar{\pi}_n\Rightarrow
  \delta_{q^*}$, where $q^*$ is the unique maximizer of the
  following function:
  \begin{equation}\label{measure2} -\langle
   \phi,q\rangle + \sum_{k=1}^K\alpha_k \ln q(k)
  \end{equation} Then
  $Q_n\Rightarrow \delta_{q^*}$ and (\ref{fin-dim}) holds
 with
$$P^*(x_1,\ldots,x_m)=\prod_{i=1}^m q^*(x_i).$$
\end{description}
\end{theorem}
Let us start with proving the uniqueness of the solutions of
(\ref{measure1}) and (\ref{measure2}). The proof of the following
lemma is in appendix.
\begin{lemma}\label{sol}\hfil\break
\begin{description}
  \item[1)] The function (\ref{measure1}) has an unique maximizer
  $q^*$, where
\begin{equation}\label{solla1}
q^*(k)={\a_k\over (1+|\a|)-{w(k)\over \theta}}, \quad k=1,\ldots,K
\end{equation}
where $\theta>0$ is a parameter satisfying $\theta=\langle w
,q^* \rangle$.
  \item[2)] The function (\ref{measure2}) has an unique maximizer $q^*$, where
\begin{equation}\label{solla2}
q^*(k)={ \alpha_k \over \phi(k)+|\a|-\theta},\quad k=1,\ldots,K,
\end{equation}
where $\theta>0$ is the parameter satisfying $\theta=\langle
\phi,q^*\rangle.$
\end{description}
\end{lemma}
\paragraph{Proof of theorem \ref{thm2}.}
\begin{description}
  \item[1)] In the case $\lambda=0$, the measure $\bp$ has the following
  density with respect to the Lebesgue measure:
$$\bp(q)={1\over Z_n }\langle e^{-\phi} ,q \rangle^{n}\cdot {1\over B(n\a)}\prod_{k=1}^K
(q(k))^{n(\alpha_k-{1\over n})}={f_n(q)^n\over Z'_n},$$ where
$$f_n(q):=\langle e^{-\phi} ,q \rangle\prod_{k=1}^K(q(k))^{(\alpha_k-{1\over
n})},\quad Z'_n:=\int f^n_n(q)dq.$$ Clearly for every $q$,
$f_n(q)\to f(q)$, where
$$f(q)=\langle e^{-\phi} ,q \rangle\prod_{k=1}^K(q(k))^{\alpha_k}.$$
It is not hard to see that the convergence is uniform, i.e.
$\sup_q |f_n(q)-f(q)|\to 0$. By {\bf 1)} of Lemma \ref{sol}, the
 function $f$ has unique maximizer $q^*$  (\ref{solla1}), i.e.
${{\cal P}^*}=\{q^*\}$.  Now apply Proposition
\ref{point2} with
 $\pi$ being the Lebesgue measure on $\P$ (hence $\pi$ is
finite) and $m_n=n$ so that
$$\nu_n(E)=\int_E h_n d q={1\over Z'_n} \int_E {f_n(q)^n }dq=\bp(E).$$
Since all assumptions are fulfilled, we have $\bp\Rightarrow
\delta_{q^*}$. Since $w_n=w$, by Lemma \ref{lemma1} the limit
process has finite-dimensional distributions
$$P^*(x_1,\ldots,x_m)=\prod_{i=1}^mr^*(x_i),\quad \text{where}\quad r^*=r_{q^*}$$
so that the limit process $P^*$ corresponds to a i.i.d. sequence
$X_1,X_2,\ldots$  with $X_1\sim r^*$. According to Lemma
\ref{lemma2}, the frequencies $Q_n$ converge weakly to the
measure ${r^*}$ and this is also quite obvious by SLLN.
\item[2)] The proof is similar: $\bp$ has density (with respect to
Lebesgue measure)
$${f_n(q)^{m_n}\over Z'_n},\quad \text{where}\quad f_n(q)=\langle e^{-{\phi\over n^{\lambda}}} ,q \rangle^{n^{\lambda}}\prod_{k=1}^K
(q(k))^{(\alpha_k-{1\over n^{1-\lambda}})},\quad
m_n=n^{(1-\lambda)}.$$ Since
$$\langle e^{-{\phi\over n^{\lambda}}} ,q \rangle^{n^{\lambda}}\to
e^{\langle \phi,q\rangle },$$ uniformly over $q$, we have that
sequence $f_n$ converges uniformly to $$f(q)=e^{-\langle
\phi,q\rangle }\prod_{k=1}^K (q(k))^{\alpha_k}.$$ By {\bf 2)} of
Lemma \ref{sol}, the function $f$ has unique maximizer $q^*$
 (\ref{solla2}). As in the case {\bf 1)}, it is easy to see that
 the  assumptions of Proposition \ref{point2} are fulfilled with
$\pi$ being Lebesgue measure on $\P$, $m_n=n^{1-\lambda}$ and so
$\bp\Rightarrow \delta_{q^*}$. In the present case, for every
$k=1,\ldots,K$, $w_n(k)\to 1$ and so by Lemma \ref{lemma1}, the
limit process $P^*$ is i.i.d. process with distribution ${q^*}$
(because $r$ is identity function). According to Lemma
\ref{lemma2}, the frequencies $Q_n$ converge weakly to the
measure $\delta_{q^*}$.
\end{description}
\subsection{Relation between $\lambda=0$ and $\lambda\in (0,1)$}
From (\ref{exp}), it follows:
\begin{equation}\label{exp2}
\sup_{q\in {\cal P}}\Big|  m\ln \langle \exp[-{\phi\over m}], q
\rangle +\langle \phi, q \rangle]\Big|\to 0
\end{equation}
so that with
$$f_m(q):=\ln \langle \exp[-{\phi\over m}], q
\big\rangle+\sum_k {\a_k\over m} \ln q(k), \quad f(q)=-\langle \phi,
q \rangle+\sum_k {\a_k} \ln q(k),$$ we have
$$\sup_{q\in {\cal P}}|m f_m(q)-f(q)|\to 0.$$
Since $f(q)$ is as in (\ref{measure2}), it has the unique  maximizer
$q^*$ given in (\ref{solla1}). On the other hand, the maximizer of
$m f_m(q)$ is the same as the maximizer of $f_m(q)$, which
corresponds to (\ref{measure1}) where $\phi$ is replaced by $\phi/m$
and $\a$ is replaced by $\a/m$. Let this unique maximizer be $q^*_m$
Since the functions $mf_m(\cdot)$ and $f(\cdot)$ are continuous,
uniformly convergent and having unique maximum, it follows that
$q^*_m \to  q^*$ (in usual sense, because ${\cal P}$ is compact).
Thus, we have proven the following proposition.
\begin{proposition}\label{propa} Let
$$q^*_m=\arg\max_q \Big(\ln \langle \exp[-{\phi\over m}], q
\big\rangle+\sum_k {\a_k\over m} \ln q(k)\Big)$$ and let $q^*$ be
the maximizer of (\ref{measure2}). Let $r_m^*$ be the corresponding
$r$ measure, i.e. $r_m^*(k)\propto q^m(k)\exp[-{\phi(k)\over m}].$
Then $q^*_m\to q^*$ and $r^*_m\to q^*$. \end{proposition}
\subsection{Product of Dirichlet priors}\label{sec:prod}
Recall the setup in Subsection \ref{sec:product}. The set of genomes
is now $\X^L=\overbrace{\X\times \cdots \times \X}^L$ and the
breeding process $\xi=(\xi^1,\ldots,\xi^L)$, where $\xi^l$ are
independent exchangeable processes. We now assume that  the prior of
$\xi^l$ is $\pi^l={\rm Dir}(\a^l),$  where
$\a^l=(a_1^l,\ldots,\a_K^l)$, $l=1,\ldots,L$. In this model, $L$
different Polya urns are run independently. Let ${\cal P}^L$ be the
set of $L$-fold product measures: $${\cal P}^L:=\{q^1\times \cdots
\times q^L: q^j\in \P\},$$ where $\P$, as previously, stands for the
$(K-1)$-dimensional simplex. Observe that $\P^L$ is a compact subset
of the set of all possible probability measures on $\X^L$. Since the
components of $\xi$ are independent, the prior $\pi$ of $\xi$ is the
product of Dirichlet measures $\pi=\pi^1\times \cdots \times \pi^L$.
This means that the support of $\pi$ is $\P^L$ and for every element
$q=q^1\times \cdots\times q^L\in \P^L$, the density is (with slight
abuse of notation, $\pi$ stands for the measure as well as for its
density)
$$\pi(q)=\prod_{l=1}^L\pi^l(q^l)={1\over B}\prod_{l=1}^L \prod_{k=1}^K
\big(q(k)^l\big)^{\a^l_k-1},\quad B:=\prod_{l=1}^L B(\a^l).$$
The function $\phi$ is now defined on the set $\X^L$, and so for any
$q\in \P^L$, $$\langle \phi,q \rangle=\sum_{(k_1,\ldots,k_L)\in
\X^L}\phi(k_1,\ldots,k_L)q^1(k_1)\cdots q^L(k_L)$$ and $\langle
e^{-\phi} ,q \rangle$ is defined similarly. When $\lambda=0$, the
measure $\bp$ has density $f_n(q)^n/Z'_n$, where
$$f_n(q)=\langle e^{-\phi} ,q \rangle \prod_{l=1}^L\prod_{k=1}^K(q^l(k))^{(\alpha^l_k-{1\over
n})},\quad Z'_n:=\int f^n_n(q)dq.$$ Clearly $f_n(q)$ converges
uniformly to
\begin{equation}\label{eq1}
f(q):=\langle e^{-\phi} ,q \rangle
\prod_{l=1}^L\prod_{k=1}^K(q^l(k))^{(\alpha^l_k)}.\end{equation}
Similarly, when $\lambda\in (0,1)$ the measure $\bp$ has density
$f_n(q)^{m_n}/Z'_n$, where $m_n=n^{1-\lambda}$,
$$f_n(q)=\langle e^{-\phi\over n^{\lambda}} ,q \rangle^{n^{\lambda}}\prod_{l=1}^L\prod_{k=1}^K(q^l(k))^{(\alpha^l_k-{1\over
n^{1-\lambda}})},\quad Z'_n:=\int f^{m_n}_n(q)dq.$$ Again, $f_n(q)$
converges uniformly to
\begin{equation}\label{eq2}
f(q):=e^{\langle {-\phi} ,q \rangle}
\prod_{l=1}^L\prod_{k=1}^K(q^l(k))^{(\alpha^l_k)}.\end{equation} When
(\ref{eq1}) (resp. (\ref{eq2})) have unique maximizer $q^*$, then the
statements of Theorem \ref{thm2} hold (the proof is the same):
\begin{description}
  \item[1)]  Suppose $\lambda=0$ and (\ref{eq1}) has unique maximizer
$q^*=q^1\times \cdots\times q^L$. Then $\bp\Rightarrow
\delta_{q^*}$, and $Q_n\Rightarrow \delta_{r^*}$, where
$$r^*(k_1,\ldots,k_L)\propto w(k_1,\ldots,k_L)q^1(k_1)\cdots
q^L(k_L),$$ and $w(k_1,\ldots,k_L)=\exp[-\phi(k_1,\ldots,k_L)].$
Observe that the measure $r^*$ is not necessarily a product measure.
Then also (\ref{fin-dim}) holds with
$$P^*(x_1,\ldots,x_m)=\prod_{i=1}^mr^*(x_i),\quad x_i\in \X^L.$$
  \item[2)]  Suppose $\lambda\in (0,1)$ and (\ref{eq2}) has unique maximizer
$q^*=q^1\times \cdots\times q^L$. Then $\bp\Rightarrow
\delta_{q^*}$, $Q_n\Rightarrow \delta_{q^*}$ and (\ref{fin-dim})
holds with
$$P^*(x_1,\ldots,x_m)=\prod_{i=1}^mq^*(x_i),\quad x_i\in \X^L.$$
\end{description}
In the case $L>1$, the maximizer of (\ref{eq1}) and (\ref{eq2}) is
not always unique. Whether it is unique or not depends on $\phi$ and
vectors $\alpha^l$. Indeed, maximizing (\ref{eq1}) is equivalent to
minimizing
\begin{equation}\label{eq1ln}
-\ln \big(\langle e^{-\phi} ,q \rangle \big)-\sum_{l=1}^L\sum_{k=1}^K \alpha_k^l \ln (q^l(k)),
\end{equation}
and $-\sum_{l=1}^L\sum_{k=1}^K \alpha_k^l \ln (q^l(k))$ is always a
convex function. So, when the parameters $\alpha^l_k$ are big
enough, then the whole function (\ref{eq1ln}) becomes convex. The
same argument holds for (\ref{eq2}). We shall present some
sufficient conditions for convexity of (\ref{eq1}) and (\ref{eq2})
for the case $K=L=2$ below.

Recall: when a positive continuous function $f(q)$ has one maximizer
$q^*$, then for any sequence $m_n\to \infty$, the measures $\nu_n$
with densities proportional to $f^{m_n}(q)$ converge weakly to
$\delta_{q^*}$. When the function has, say, two maximizers, $q^*_1$
and $q^*_2$, then by Proposition \ref{point2}, for all disjoint open
balls $B_1$ and $B_2$ so that $q^*_i\in B_i$, it still holds that
$\nu_n(B_1)+\nu_n(B_2)\to 1$. Thus, when the measures $\nu_n$ are
weakly convergent, then the limit measure is concentrated on $\{
q^*_1,q^*_2\}$ so that the limit measure must be $p
\delta_{q^*_1}+(1-p)\delta_{q^*_2}$, for some $p\in[0,1]$. In this
case $\nu_n(B_1)\to p$. However, the function $f$ might be so that
the limits $\nu_n(B_i)$ do not exist.  And even if they do exist
(i.e. the measures $\nu_n$ are weakly convergent), the limits $p$
and $1-p$ might be arbitrary real numbers, and hard to determine.
Therefore the following theorem adapted from \cite{HaarioSaksman}
(Theorem 5.7 and a remark after it) might be very useful.
\begin{theorem}\label{thmg} Suppose $K\subset \mathbb{R}^k$ is a compact non-empty subset
and let $g: K\to [0,\infty)$ be a twice continuously differentiable
function with finitely many minimum points $\{a_1,\ldots,a_r\}$ all
located in the interior of $K$. Let, for every $i=1,\ldots r$ the
Hessian of $g$ at $a_i$ be positive definite. Given any increasing
sequence $m_n$, define the sequence of measures
$$\nu_n(E):=  {1\over Z_n} \int_E  \exp[-m_n\cdot g(x)]dx,\quad Z_n:=\int \exp[-m_n\cdot
g(x)]dx.$$ Then $\nu_n\Rightarrow \nu$, where
$$\nu=\sum_{i=1}^r p_i \delta_{a_i},\quad \text{  with }
p_i\propto {1\over \sqrt{ \det H(a_i)}}$$
and $\det H(a_i)$ is a determinant of Hessian evaluated at
$a_i$.\end{theorem}
To apply the theorem in our case, let us first note that any $K-1$
dimensional simplex can be considered as a $K-1$-dimensional
non-empty compact set $$\P_K=\{(q_1,\ldots, q_{K-1}: q_k\geq 0,
\sum_k^{K-1} q_k\leq 1\}.$$ Therefore, our search space  $\P^L$ can
be considered as a subset in  $\mathbb{R}^{L(K-1)}.$ This subset has
non-empty interior. Clearly  any solution of (\ref{eq1}) and
(\ref{eq2}) has all components strictly positive  so that all
maximizers of maximizer of (\ref{eq1}) are interior points and the
same holds for (\ref{eq2}). We take $g(q)=-\ln f(q)$, where $f$ is
as in (\ref{eq1}) or (\ref{eq2}). Thus, for any $m$, $\exp[-m
g(y)]=f^m(q)$ so that the measure defined in the statement of
theorem is
$$\nu_n(E)\propto \int_E f^{m_n}(q)dq.$$
However, even when the measures $\nu_n$ converge weakly to a limit,
it does not automatically follow that  the measures $\bp$ converge
to the same limit even if $f_n$ converges to $f$ uniformly. This
convergence might depend on the speed of the uniform convergence,
and we leave it for the further studies and proceed with an example
instead.
\subsubsection{The case $K=L=2$}
Let us analyze more closely the case  $K=2$ and $L=2$. Denote
$q^1(1)=:z_1$ and $q^2(1)=:z_2$. Also denote $\a_k^1=\a_k$ and
$\a_k^2=\beta_k$. The function (\ref{eq1ln}) is
\begin{equation}\label{g1}
g(z_1,z_2)=-\ln \big(\langle w,z
\rangle\big)-\alpha_1 \ln z_1- \alpha_2 \ln (1-z_1) - \beta_1 \ln
z_2- \beta_2 \ln (1-z_2),\end{equation} where
$$\langle w,z \rangle=
w(1,1)z_1z_2+w(1,2)z_1(1-z_2)+w(2,1)(1-z_1)z_2+w(2,2)(1-z_1)(1-z_2).$$
Thus with $w^*:=w(1,1)-w(1,2)-w(2,1)+w(2,2)$ and
\begin{align*}
\theta^1_1&:=w(1,1)z_2+w(1,2)(1-z_2),\quad \theta^1_2=w(2,1)z_2+w(2,2)(1-z_2)\\
\theta^2_1&:=w(1,1)z_1+w(2,1)(1-z_1),\quad
\theta^2_2:=w(1,2)z_1+w(2,2)(1-z_1).
\end{align*}
we obtain the Hessian

\begin{align*}
\left(
  \begin{array}{cc}
    {\partial^2 g\over \partial z_1^2}  &  {\partial^2 g\over \partial z_1 \partial z_2} \\
    {\partial^2 g\over \partial z_2 \partial z_1} &  {\partial^2 g\over \partial z_2^2} \\
  \end{array}
\right)=\left(
          \begin{array}{cc}
            {\alpha_1\over z_1^2}+{\alpha_2\over (1- z_1)^2} +{( \theta_1^1-\theta_2^1 )^2\over \langle w,z \rangle^2} &   {(\theta_1^1-\theta_2^1)(\theta_1^2-\theta_2^2)-w^*\langle w,z \rangle\over
\langle w,z \rangle^2}\\
{(\theta_1^1-\theta_2^1)(\theta_1^2-\theta_2^2)-w^*\langle w,z
\rangle\over
\langle w,z \rangle^2}          & {\beta_1\over z_2^2}+{\beta_2\over (1-z_2)^2}+{(\theta_1^2-\theta_2^2)^2\over \langle w,z \rangle^2} \\
          \end{array}
        \right).
\end{align*}
The  elements in the main diagonal are strictly positive and
therefore the matrix is positive definite if the determinant is
positive, i.e.
\begin{align*}
\Big({\alpha_1\over z_1^2}+{\alpha_2\over (1- z_1)^2} +{( \theta_1^1-\theta_2^1 )^2\over \langle w,z \rangle^2}\Big)
\Big({\beta_1\over z_2^2}+{\beta_2\over (1-z_2)^2}+{(\theta_1^2-\theta_2^2)^2\over \langle w,z \rangle^2}\Big)
>\Big({(\theta_1^1-\theta_2^1)(\theta_1^2-\theta_2^2)\over
\langle w,z \rangle^2}-{w^*\over \langle w,z \rangle}\Big)^2.
\end{align*}
Observe: $ \langle w,z \rangle \geq \min_{i,j}w(i,j)>0$ and so
$$\big|{w^*\over \langle w,z \rangle}\big|\leq {|w^*|\over \min_{i,j}w(i,j)}.$$
On the other hand, for any $(z_1,z_2)\in [0,1]\times[0,1]$,
$${\alpha_1\over z_1^2}+{\alpha_2\over (1-z_1)^2}\geq (\a_1^{1\over
3}+\a_2^{1\over 3})^3, \quad {\beta_1\over z_2^2}+{\beta_2\over
(1-z_2)^2}\geq (\beta_1^{1\over 3}+\beta_2^{1\over 3})^3.$$
Therefore, it is not hard to see that when
\begin{equation}\label{posd1}
(\a_1^{1\over 3}+\a_2^{1\over 3})^3(\beta_1^{1\over
3}+\beta_2^{1\over 3})^3>\big({w^*\over
\min_{i,j}w(i,j)}\big)^2,\end{equation} then the Hessian is always
positive definite, i.e. the function $g$ in (\ref{g1}) is strictly
convex, and the minimum unique. In particular, the condition holds
if $w^*=0$. In particular, if $\a_1=\beta_1$, $\a_2=\beta_2$ and
$w(1,2)=w(2,1)$, then under (\ref{posd1})  the unique solution
$z_1,z_2$ satisfies $z_1=z_2$ (by symmetry). Indeed, in symmetric
case it holds: if $(z_1,z_2)$ is a solution, then so must be
$(z_2,z_1)$, and if the solution is unique, then it must be that
$z_1=z_2$. If (\ref{posd1}) fails, then the function might be
non-convex, and for small $\alpha_k$ and $\beta_k$ values it is
(given $w^*\ne 0$), but evaluated at the minimums, the Hessian might
still be positive definite and so Theorem \ref{thmg} might apply.

Similarly, for (\ref{eq2})
\begin{equation}\label {g2}
g(z)=-\ln f(z)= \langle \phi, z \rangle -\alpha_1^1 \ln z_1-
\alpha_2^1 \ln (1-z_1) - \alpha_1^2 \ln z_2- \alpha_2^2 \ln
(1-z_2)\end{equation} and so the Hessian is
\begin{align*}
\left(
  \begin{array}{cc}
    {\partial^2 g\over \partial z_1^2}  &  {\partial^2 g\over \partial z_1 \partial z_2} \\
    {\partial^2 g\over \partial z_2 \partial z_1} &  {\partial^2 g\over \partial z_2^2} \\
  \end{array}
\right)=\left(
          \begin{array}{cc}
            {\alpha_1\over z_1^2}+{\alpha_2\over (1- z_1)^2} & \phi^* \\
\phi^*          & {\beta\over z_2^2}+{\beta_2\over (1-z_2)^2} \\
          \end{array}
        \right),
\end{align*}
where
$$\phi^*:=\phi(1,1)-\phi(1,2)-\phi(2,1)+\phi(2,2).$$
Since the elements of the main diagonal are are strictly positive,
the matrix is positive definite if and only if the determinant is
positive:
\begin{equation}\label{pdf}
\big({\alpha_1\over z_1^2}+{\alpha_2\over (1-
z_1)^2}\big)\big({\beta_1\over z_2^2}+{\beta_2\over
(1-z_2)^2}\big)>(\phi^*)^2.\end{equation}
Thus, when the following inequality holds
\begin{equation}\label{posdef2}
(\a_1^{1\over 3}+\a_2^{1\over 3})^3(\beta_1^{1\over
3}+\beta_2^{1\over 3})^3>(\phi^*)^2,
\end{equation}
then the Hessian is always positive definite and the function
(\ref{g2}) strictly convex implying that the minimum is unique. If
the Hessian is not always unique, but (\ref{pdf}) holds for
minimums, then Theorem \ref{thmg} applies. Again, when $\a=\beta$
and $\phi(1,2)=\phi(2,1)$, then under (\ref{posdef2}) the unique
minimum is such that $z_1=z_2$. For example, when $\a=\beta=(2,2)$
and $\phi(1,1)=1, \phi(1,2)=\phi(2,1)=2,\phi(2,2)=3$, then
$\phi^*=0$ and, therefore, (\ref{posdef2}) holds. It means the
minimum is unique, $z_1=z_2$, and one can verify that
$$z_1=z_2={\sqrt{17}-3\over 2}\approx 0.561.$$ So the
unique limit distribution in this case is $q\times q$, where
$q=(z,1-z)$. But when $\a=(2,3)$, $\beta=(3,2)$ and $\phi$ is as
previously, then the minimum is again unique but since $\a\ne
\beta$, we now have $z_1\ne z_2$: $z_1=\sqrt{6}-2\approx 0.45,\quad
z_2=\sqrt{7}-2\approx 0.645.$ This means: the unique limit
distribution is $q^1\times q^2$, where
$q^1=(\sqrt{6}-2,3-\sqrt{6})$, $q^2=(\sqrt{7}-2,3-\sqrt{7})$.

When $\a=\beta=(2,2)$, $\phi(1,1)=4,
\phi(1,2)=\phi(2,1)=2,\phi(2,2)=4$, then $\phi^*=4$, but
(\ref{posdef2}) still holds and therefore there is unique minimum:
$z_1=0.5,z_2=0.5$. However, when the $\phi$ is as previously, but
 $\a=\beta=(0.25,0.25)$, then (\ref{posdef2})
fails. It turns out that now the function is not convex and  there
are two minima:
$$(z_1={2-\sqrt{2}\over 4},z_2={2+\sqrt{2}\over 4}),\quad
(z_1={\sqrt{2}+2\over 4},z_2={\sqrt{2}-2\over 4})$$ Observe that
$z_2=1-z_1$. Thus the limit measures are $q^1\times q^2$ and
$q^2\times q^1$, where $q^1=(z_1,z_2)$ and $q^2=(z_2,z_1)$. These
two product measures   are different. Finally observe that in both
cases (\ref{pdf}) holds, so that by Theorem \ref{thmg},
$$\nu_n\Rightarrow {1\over 2}\delta_{q^1\times q^2}+ {1\over
2}\delta_{q^2\times q^1}.$$ The function
$$f_n(q)=\langle e^{-{\phi\over n^{\lambda}}},q\rangle^{n^{\lambda}}
\prod_{l=1}^2\prod_{k=1}^2 (q(k)^l)^{0.25-{1\over n^{1-\lambda}}}$$
is symmetric, i.e. $f_n(z_1,z_2)=f_n(z_2,z_1)$ and then $\bp
\Rightarrow {1\over 2}\delta_{q^1\times q^2}+ {1\over
2}\delta_{q^2\times q^1}$. Since now $r_q=q$, by Lemma \ref{lemma1},
(\ref{fin-dim}) holds, with
$$P^*(x_1,\ldots,x_m)= {1\over 2}\prod_{i=1}^mq^1\times q^2(x_i)+{1\over 2}\prod_{i=1}^mq^2\times q^1(x_i),\quad x_i\in \X^L.$$

\subsection{An application: exact calculation of stationary distributions for epistatic fitness}
Exchangeable sampling with a product of Dirichlet priors   is an abstract model of sexual reproduction, with the constraint that the population is in linkage equilibrium. For a considerable class of fitness functions
it is possible to compute the stationary distribution exactly, so that the model is potentially a tool for the investigation of the properties of genetic architectures. This line of research is beyond the scope of this paper, but we give one example of computation of the stationary distribution for a non-trivial fitness function in a multi-locus system.

Consider a the case where $\mathcal{X}=\{1,2\}^L$, and the fitness of a genome $x = (x^1, \ldots, x^L)$ is defined as
\begin{equation*}
w(x^1, \ldots, x^L) =
\begin{cases}
1 + s & \text{if $x^1=\cdots=x^L=1$}\\
1 & \text{otherwise}
\end{cases}
\end{equation*}
In other words, if a genome $(x^1, \ldots,x^L)$ is `perfect', consisting only of the preferred alleles (represented as 1s), then it has a fitness advantage of $s$, but if even a single element is imperfect, then there is no fitness advantage.  It seems plausible that there are  that there are combinations of well-separated alleles in biology that have a fitness advantage only if all alleles in the set are of the correct type; this is a simple model of such a situation.

To keep the notation simple, we assume that all independent components of breeding process $\xi^l$,  $l=1,\ldots,L$ have the same Dirichlet prior with parameter  $\alpha=(\alpha_1,\ldots,\alpha_K)$. We shall denote the law of the  breeding process by   $P_{\xi}(\cdot|L,\alpha)$. When $(\xi_1,\ldots,\xi_n)\sim P_{\xi}(\cdot|L,\alpha)$, then the probability that the number of  prefect genomes (vectors $\xi_i$ consisting of ones, only) equals to  $n_c$ is
\begin{equation*}
P_\xi(n_c, n-n_c; L, \alpha) := P\left( \sum_{i=1}^n I_{1+s}(w(\xi_i)) = n_c \right).
\end{equation*}
 In words, $P_\xi( n_c, n-n_c; L, \alpha)$ denotes the probability of $n_c$ perfect genomes and $n-n_c$ imperfect genomes in a population of $n$ genomes, each of length $L$, sampled from a product of $L$ Polya urn processes, each with prior parameter $\alpha$.

We calculate $P_\xi( n_c, n-n_c ; L, \alpha)$ using recursion on $L$. When $L=1$, $n_c$ is simply the number of 1s in the population of $n$, so
\begin{equation*}
P_\xi(n_c, n-n_c; 1, \alpha) = \binom{n}{n_c} \frac{ (\alpha_1)_{n_c} (\alpha_2)_{n-n_c} }{(\vert\alpha\vert)_n}.
\end{equation*}
It is not hard to show, using independence of $\xi^1, \ldots, \xi^L$ and exchangeability, that
\begin{equation*}
P_\xi( n_c, n-n_c ; L, \alpha) = \sum_{m=n_c}^N P_\xi( m, n-m; L-1, \alpha) P_\xi(n_c, m-n_c; 1, \alpha)
\end{equation*}
Finally, the stationary distribution of the number of perfect genomes in a population of $n$ is:
\begin{equation}
P_n( n_c, n-n_c; L, \alpha) = \frac{1}{Z(n,L,\alpha)} P_\xi(n_c, n-n_c; 1, \alpha)  (1+s)^{n_c}
\end{equation}
where the normalising constant is
\begin{equation*}
Z(n,L,\alpha) = \sum_{n_c=0}^n P_\xi(n_c, n-n_c; 1, \alpha)  (1+s)^{n_c}.
\end{equation*}
Thus in our model the stationary distribution of the number of perfect genomes is simple to derive without making any approximations. In general, because the processes $\xi^1, \ldots,\xi^L$ are independent, and as a result of the factorisation of the stationary distribution in equation \ref{eq:stationary1},  the stationary distribution for a fitness function $w$ can be derived exactly using the sum-product algorithm if $w$ can be represented as a tractable factor graph \cite{kschischang2001factor} over the $L$ urn distributions. This introduces a new method for studying the effect of epistatic fitness on population structure.

\section{Experiments}
\label{sec:experiments} Let $K=2$, $\phi(1)=0$, $\phi(2)=\ln 6$,
$\alpha_1=0.3$, $\alpha_2=0.7$. Let us find the limit measures $q^*$
as in Theorem \ref{thm2} in the following  cases: $\lambda=0$,
$\lambda\in (0,1)$ and $\lambda=1$.
\paragraph{\bf Case $\lambda=0$:} Then, as it can be easily checked by
verifying (\ref{solla1}) that  $q^*=({3/5},2/5)$. Since
$\theta=\langle q^*,w \rangle=2/3$,  the measure $r^*$ is as
  follows: $r^*_1=w(1)q^*(1)/\theta=9/10$ and
  $r^*(2)=w(2)q^*(2)/\theta=1/10$. Therefore the limit process
  governed by $P^*$ is an i.i.d. process with measure $r^*$, and
  so due to the weight function, the {average} proportion of the first
  genotype has increased from 0.3 (according to the prior) to
  0.9. Figure \ref{fig:graphlambda0} illustrates the convergence.
\begin{figure}[h]
\begin{center}
\includegraphics[width=10cm,height=6cm]{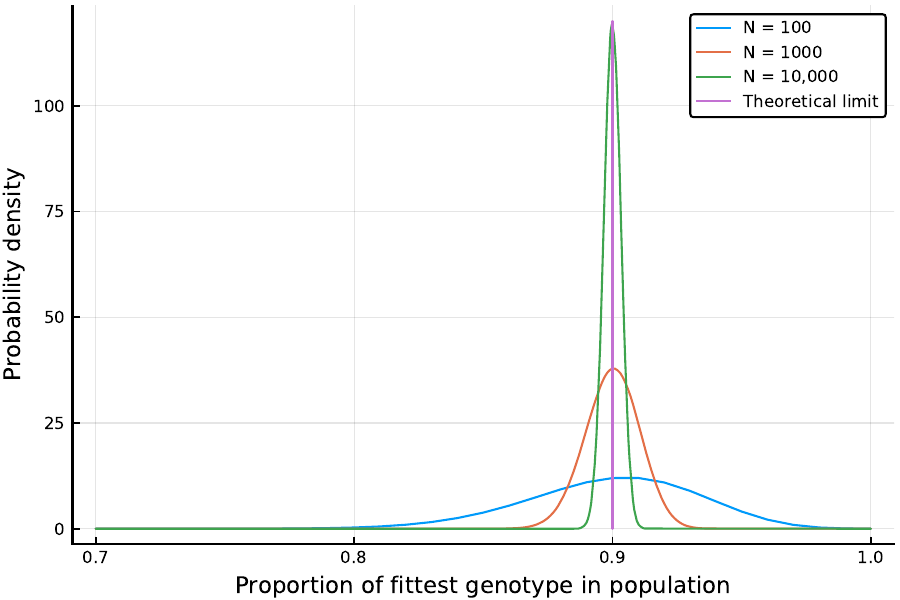}
\caption{\small Density histogram of the fraction of the first type in the population, for a Dirichlet prior $\alpha=(0.3, 0.7)$,
fitness $\phi=(0,\ln 6)$, and $\lambda=0$, and three population sizes $10^2$, $10^3$, and $10^4$. The histograms were constructed by
recording the fraction of the first type in the population over $10^8$ MCMC samples according to the process described in section\ \ref{sec:inversefitnessselection} }.
\label{fig:graphlambda0}
\end{center}
\end{figure}
\paragraph{\bf Case $\lambda\in (0,1)$:} The solution of  (\ref{solla2})
  is
\begin{equation}\label{lumi}
q^*(1)={\ln (6)-1+\sqrt{(\ln (6)-1)^2+1.2 \ln(6)}\over 2\ln
(6)}\approx 0.686. \end{equation}
 Now $r^*=q^*$, thus
we see that the limit process governed by $P^*$ is an i.i.d. process
with measure $q^*$, and so due to the weight function, the
  proportion of first genotype has increased from 0.3 (according
  to the prior) to 0.689. The increase is smaller than in the
  previous case. Figures \ref{fig:graphlambda25},  \ref{fig:graphlambda50} and \ref{fig:graphlambda75} illustrates the convergence for $\lambda=0.25,0.5,0.75$, respectively.
  We see that although the limit is the same, the speed of
  convergence depends very much on $\lambda$.
\begin{figure}[hp]
\begin{center}
\includegraphics[width=9cm,height=4 cm]{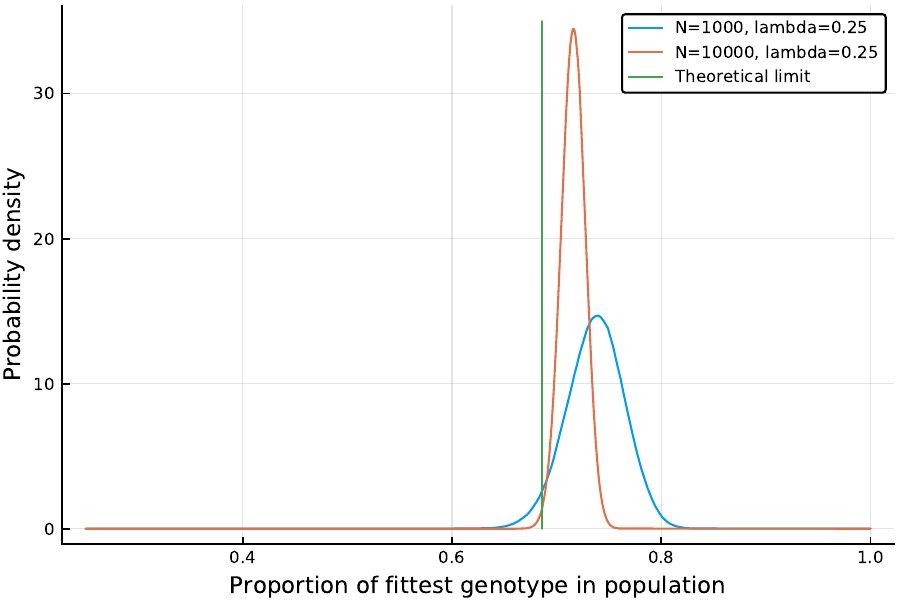}
\caption{\small Case $\lambda=0.25$ : density histogram of the fraction of the first type in the population, for a Dirichlet prior $\alpha=(0.3, 0.7)$, fitness $\phi=(0,\ln 6)$, and $\lambda=0.25$, and  population sizes $10^3$ and $10^4$. The histograms were constructed by recording the fraction of the fitter allele in the population over $10^8$ MCMC samples according to the process described in section\ \ref{sec:inversefitnessselection} }.
\label{fig:graphlambda25}
\end{center}
\begin{center}
\includegraphics[width=9cm,height=4 cm]{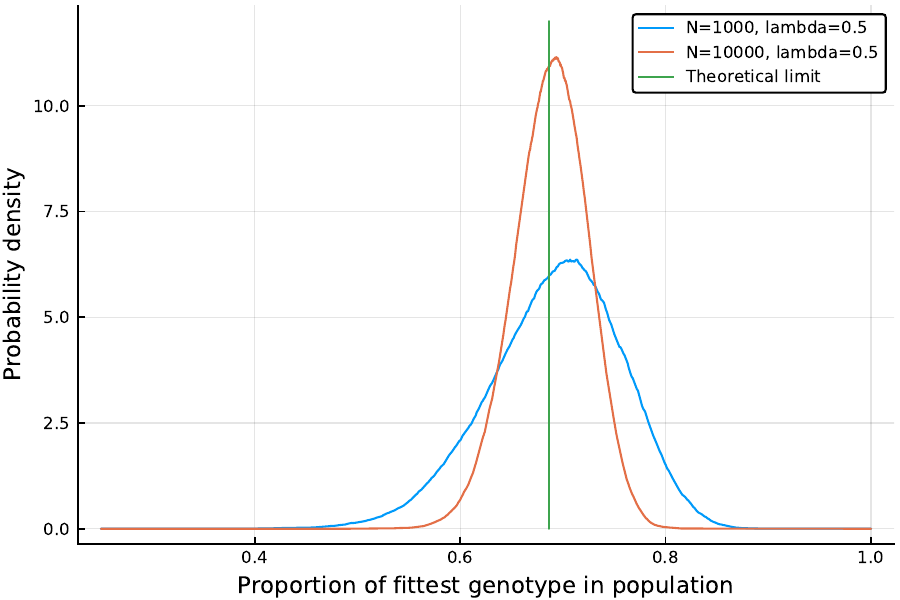}
\caption{\small Case $\lambda=0.5$ : density histogram of the fraction of the first type in the population, for a Dirichlet prior $\alpha=(0.3, 0.7)$, fitness $\phi=(0,\ln 6)$, and  population sizes $10^3$ and $10^4$. The histograms were constructed by recording the fraction of the fitter allele in the population over $10^8$ MCMC samples according to the process described in section\ \ref{sec:inversefitnessselection} }.
\label{fig:graphlambda50}
\end{center}
\begin{center}
\includegraphics[width=9cm,height=4 cm]{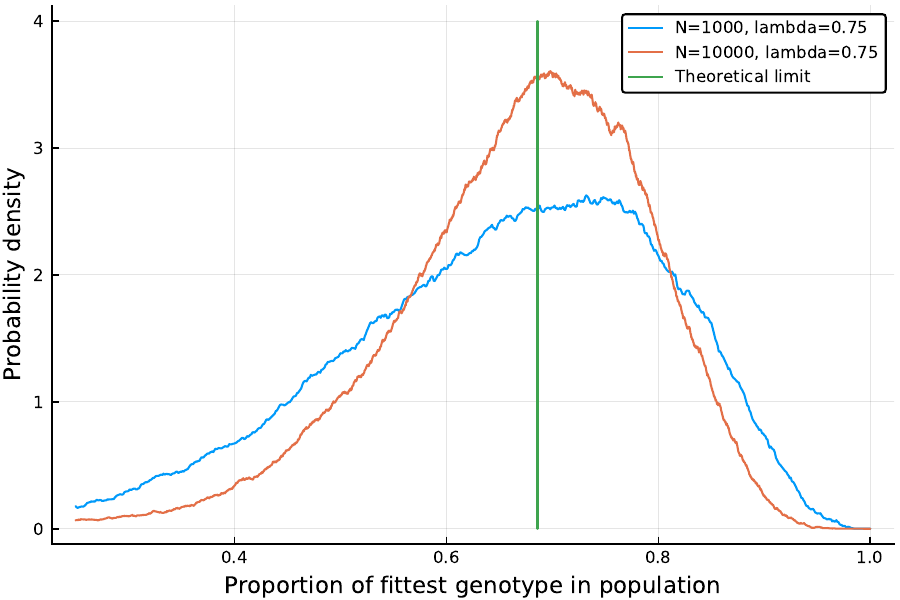}
\caption{\small Case $\lambda=0.75$ : density histogram of the fraction of the first type in the population, for a Dirichlet prior $\alpha=(0.3, 0.7)$, fitness $\phi=(0,\ln 6)$, and  population sizes $10^3$ and $10^4$. The histograms were constructed by recording the fraction of the fitter allele in the population over $10^{10}$ MCMC samples according to the process described in section\ \ref{sec:inversefitnessselection} }.
\label{fig:graphlambda75}
\end{center}
\end{figure}
Let  $q^*_m=(q_m(1),q_m(2))$ be the maximizer of $$\ln \langle
\exp[-{\phi\over m}], q \big\rangle+\sum_k {\a_k\over m} \ln q(k).$$
In our example $q_1(1)=3/5$ (the case $\lambda=0$) and
$$q_m(1)={b_m+\sqrt{b_m^2+{12\over 10}(1+m)\big(1-(1/6)^{1\over m}\big)(1/6)^{1\over
m}}\over 2(1+m) \big(1-(1/6)^{1\over m}\big)},\text{ where
 }b_m=(m+0.3)-(1/6)^{1\over m}(1.3+m).$$
According to Proposition \ref{propa}, in the process $m\to \infty$,
$q_m(1)$ tends to $q^*(1)$ from (\ref{lumi}), and one can easily
verify that it is really so.
\paragraph{\bf Case $\lambda=1$:} According to {\bf 2)} of Theorem
\ref{thm1}, the limit measure $\bar{\pi}$ has density with respect
to the Lebesgue'i measure on $[0,1]$:
$$\bar{\pi}(q)={1\over Z}\exp[-\ln 6\cdot q] (1-q)^{0.3-1}q^{0.7-1},$$
where $Z$ is the valuer of the moment generating function of
Beta(0.7,0.3)-distributed random variable evaluated at $\ln 6$. In
the density above, $q$ stands for $q(2)$. Therefore, the limit
proportion of the second genotype in the stochastic process governed
by $P^*$ is  a random variable, its distribution has density
$\bar{\pi}(q)$ as stated above. Figure \ref{fig:graphlambda1}
illustrates the convergence.
\begin{figure}[h]
\begin{center}
\includegraphics[width=10cm,height=5 cm]{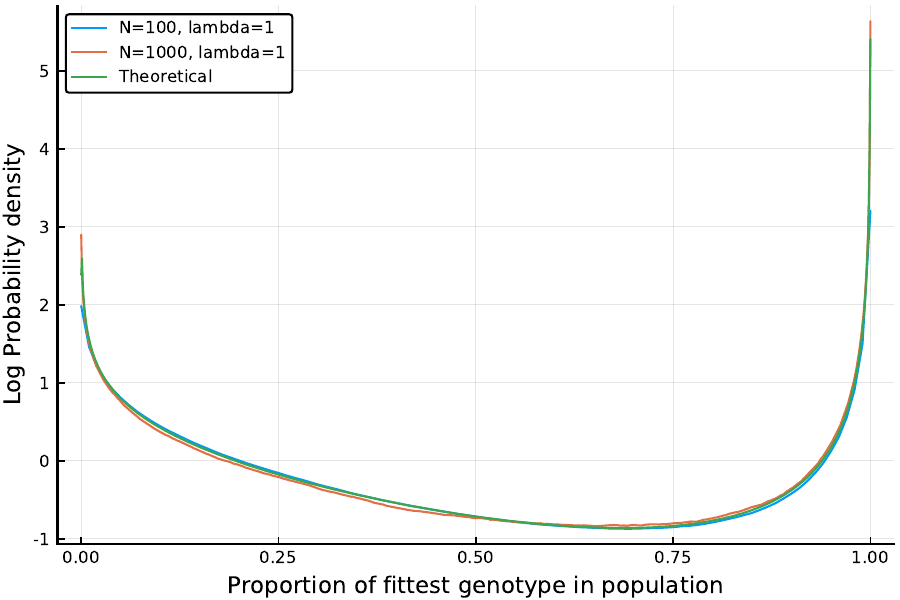}
\caption{\small Case $\lambda=1$ :  Density histogram (log scale) of the fraction of the first type in the population, for a Dirichlet prior
$\alpha=(0.3, 0.7)$, fitness $\phi=(0,\ln 6)$,  and  population sizes $10^2$ and $10^3$. Note that because the concentration parameter is kept constant, the limit distribution is a reweighted $\beta$ distribution, with infinities at 0 and 1: when $\lambda=0$, the prior is the same for all $N$, and the log-fitness term is reduced by a factor of $N$, so that $P_n$ never concentrates. The histograms were constructed by recording the fraction of the fitter allele in the population over $10^9$ and  $10^{10}$ MCMC samples respectively according to the process described in section\ \ref{sec:inversefitnessselection} }.
\label{fig:graphlambda1}
\end{center}
\end{figure}
\section{Conclusions}
\label{sec:conclusions}

We have constructed MCMC algorithms that are similar to existing
genetic algorithms. The `breeding' consists of sampling from
exchangeable distributions based on the Dirichlet distribution, and
the `selection' is essentially Metropolis-Hastings. The sequence of
populations forms a reversible Markov chain that satisfies detailed
balance conditions.  We have exhibited two possible sampling
distributions:
 more elaborate exchangeable sampling distributions are possible. The entire MCMC procedure is a population generalisation of Metropolis-Hastings.
As far as we are aware, this is the first

{implementable computational model} of sexual reproduction that
exactly satisfies detailed balance, and for which the stationary
distribution can be written in closed form for arbitrary fitness
functions.

We also explored some properties of the stationary distribution, and
showed that for any fitness function there are three non-trivial
limiting distributions for large population sizes, with two phase
transitions. This is a first step towards a more general
understanding of the interaction of the population size, fitness
scaling, and mutation rate in genetic algorithms and evolutionary
models.

Formulating a genetic model as a MCMC procedure opens a new research
direction in using the many techniques developed in MCMC to achieve
faster convergence to the stationary distribution using different
MCMC kernels.

We have shown that the stationary distribution is unaffected by multiplicative noise in fitness evaluations.
This has been suggested by, for example, \cite{morse2016simple}, but our techniques allow a  proof of this effect.

{Finally there is a more general conclusion from our analysis.   For
many years, since \cite{holland1975adaptation} and
\cite{goldberg1989genetic}, a widely suggested folk-motivation for
genetic algorithms has been that because they are inspired by
natural biological evolution, and because evolution has produced the
variety of life on earth, genetic algorithms should be in some sense
generally effective. Our analysis makes it clear that genetic
algorithms are more  closely related to conventional MCMC methods
for non-parametric Bayesian inference than has previously been
recognised.

%

\section*{Appendix}
\paragraph{Proof of claim \ref{claim}}
\begin{proof}  First note that
$$\sup_{q\in \P}|\langle w_n,q\rangle-\langle w,q \rangle|\leq \|w_n-w\|\|q\|\leq \|w_n-w\| \to 0.$$
Now use the fact that if $f_n, f, g_n,g$ are nonnegative functions
 such that $\sup_x |f_n(x)-f(x)|\to 0$, $\sup_x
|g_n(x)-g(x)|\to 0$, $\inf_x g(x)=g_*>0$ and $f$, $g$ are bounded
above, then with $g^*=\sup_x g(x)$ and $f^*=\sup_x f(x)$
\begin{align*}
\sup_x \big|{f_n(x)\over f(x)}-{g_n(x)\over g(x)}\big|&=\sup_x
\big|{f_n(x)g(x)-g_n(x)f(x)\over g_n(x)g(x)}\big|\\
&\leq \sup_x \big|{(f_n(x)-f(x))g(x)\over g_n(x)g(x)}\big|+\sup_x
\big|{(g_n(x)-g(x))f(x)\over g_n(x)g(x)}\big|\\
&\leq \sup_x \big|{(f_n(x)-f(x))g^*\over g_n(x)g(x)}\big|+\sup_x
\big|{(g_n(x)-g(x))f^*\over g_n(x)g(x)}\big|\to 0,
\end{align*}
because for $n$ big enough $g_n(x)g(x)\geq {g^2_*\over 2}$ for every
$x$. Take $x=q$, $f_n(q)=w_n(k)q(k)$, $f(q)=w(k)q(k)$,
$g_n(q)=\langle w_n,q\rangle$ and $g(q)=\langle w,q\rangle$. Then
$g_*=w(K)>0$, $f^*=g^*=w(1)$ and so (\ref{unif}) follows.\end{proof}
\subsection{Proof of Proposition \ref{point2}}
Recall that $f_n$ and $f$ are continuous and bounded functions on
$\P$ so that $\|f_n\|_{\infty}<\infty$ and $\|f\|_{\infty}<\infty$.
By assumption, $\pi$ is a finite measure. Since $f_n$ converges to
$f$ uniformly, it follows that $\|f_n-f\|_{\infty}\to 0$ and so $\|
f_n\|_{\infty}\to \| f\|_{\infty}.$ For every $m$,
$$|\|f_n\|_m-\|f\|_m |\leq \|f_n-f\|_m\leq \pi(\P)^{1\over m}\|f_n-f\|_{\infty}\to 0.$$
Since $\|f\|_{m_n}\to \|f\|_{\infty}$, we have
\begin{align*}
\big|\|f_n\|_{m_n}-\|f\|_{\infty}\big|&\leq \big|\|f_n\|_{m_n}-\|f\|_{m_n}\big|+\big|\|f\|_{m_n}-|f\|_{\infty}\big|\\
&\leq \|f_n-f\|_{m_n}+\big|\|f\|_{m_n}-|f\|_{\infty}\big|\leq  \pi(\P)^{1\over m_n}\|f_n-f\|_{\infty}+\big|\|f\|_{m_n}-|f\|_{\infty}\big|\to 0.
\end{align*}
 Now fix $\delta>0$. Since
${{\cal P}^*_{\delta}}:=\{q: f(q) >
\|f\|_{\infty}-\delta \}$, we have ${{\cal P}-{\cal
P}^*_{\delta}}=\{q: f(q)\leq
 \|f\|_{\infty}-\delta\}. $
Define $\delta':=\delta /\|f\|_{\infty}$. Then
\begin{align*}
\sup_{q\in {{\cal P}-{\cal P}^*_{\delta}}}
{f_n(q)\over \|f_n\|_{m_n}}&=\sup_{q\in {{\cal
P}-{\cal P}^*_{\delta}}}{f(q)+(f_n(q)-f(q))\over \|f_n\|_{m_n}}=
\sup_{q\in {{\cal P}-{\cal P}^*_{\delta}}}{f(q)+(f_n(q)-f(q))\over \|f\|_{\infty}}{\|f\|_{\infty}\over  \|f_n\|_{m_n}}\\
&\leq \sup_{q\in {{\cal P}-{\cal
P}^*_{\delta}}}{f(q)\over \|f\|_{\infty}}{\|f\|_{\infty}\over
\|f_n\|_{m_n}}+{\|f_n-f\|_{\infty}\over \|f_n\|_{m_n}}\leq
1-{\delta'\over 2},\end{align*} provided $n$ is big enough. Thus,
$$
\sup_{q\in {\cal P}-{\cal P}^*_{\delta}}h_n(q)\leq
\Big(1-{\delta'\over 2}\Big)^{m_n}\to 0,$$ so that $\nu_n({\cal
P}^{*}_{\delta})\to 1$. We now argue that {when
${\cal P}^*=\{q^*\}$, then} for any $\epsilon>0$ there exists
$\delta>0$ so that
\begin{equation}\label{balls}
{\cal P}^*_{\delta} \subset B(q^*,\e),
\end{equation}
where $B(q^*,\epsilon)$ is a ball in Euclidean sense.  If, for an
$\epsilon>0$, such a $\delta>0$ would not exists, then there would
exist a sequence $q_n\to q$ such that $f(q_n)\nearrow f(q)$, but
$\|q_n-q\|\geq \epsilon.$  Since $\P$ is compact, along a
subsequence $q_{n'}\to q$ and by continuity $f(q_{n'})\to f(q)$. On
the other hand $\|q-q^*\|>\epsilon$ and that would contradict the
uniqueness of $q^*$. Therefore  (\ref{balls}) holds and so  for any
$\e>0$, it holds that $\nu_n \big (B(q^*,\e)\big )\to 1$. From the
definition of the weak convergence, it now follows that
$\nu_n\Rightarrow \delta_{q^*}.$

\subsection{Proof of Lemma \ref{sol}}
\begin{description}
                                       \item[1)]
To find
\begin{equation}\label{probleem1}
q^*=\arg\max_{q\in {\cal P}} \,\,[ \ln \langle e^{-\phi} ,q
\rangle + \sum_{k}\alpha_k\ln q(k)],\end{equation} we define Lagrangian
$$L(q,\beta)=\ln  \langle e^{-\phi} ,q \rangle +
\sum_{k}\alpha_k\ln q(k)-\beta (\sum_k q(k)-1)$$(here $\beta$ is
a scalar) and maximize $L(q,\beta)$ over $q> 0$ (all entries are
positive). Taking partial derivatives with respect to $q(k)$, we
have
$${e^{-\phi(k)}\over \langle e^{-\phi} ,q \rangle}+{\alpha_k\over
q(k)}=\beta,\quad \Rightarrow \quad {e^{-\phi(k)}q(k)\over \langle
e^{-\phi} ,q \rangle}+{\alpha_k}=q(k) \beta \quad \forall k.$$
With $|\alpha|=\sum_k \alpha_k,$ we have thus $\beta=1+|\alpha|$
and so the solution $q^*$ satisfies the set of equalities
\begin{equation}\label{q3}
q^*(k)={1\over 1+|\alpha|}\Big({e^{-\phi(k)}q^*_k\over \langle
e^{-\phi} ,q^* \rangle}+\alpha_k\Big),\quad \forall k.
\end{equation}
Now with  $w(k)=e^{-\phi(k)}$ define parameter $\theta:=\langle
w ,q^* \rangle$
 and rewrite (\ref{q3}) as follows
\begin{equation}\label{theta}
q^*(k)={\alpha_k \over (1+|\alpha|)-{w(k)\over \theta}} \quad
k=1,\ldots,K.\end{equation} We see that amongst the probability
vectors satisfying $\langle w ,q^* \rangle=\theta,$ the solution
is unique. Since $\alpha_k>0$ for every $k$, it is easy to see
that there is only one parameter $\theta$ such that the right
hand side of (\ref{theta}) would be a probability measure: if
$\theta'>\theta$, then for every $k$, we have
$${\alpha_k \over (1+|\alpha|)-{w(k)\over \theta}}>{\alpha_k \over (1+|\alpha|)-{w(k)\over
\theta'}}.$$ Therefore a solution of (\ref{probleem1}) is unique
vector $q^*$  given by (\ref{theta}), where $\theta=\langle w
,q^* \rangle$.
\item[2)] To find
\begin{equation}\label{probleem2}
q^*=\arg\max_{q\in {\cal P}} [-\langle \phi,q\rangle +
\sum_{k}\alpha_k \ln q(k)],\end{equation} we define Lagrangian
$$L(q,\beta)= -\langle \phi,q\rangle +
\sum_{k}\alpha_k\ln q(k)-\beta (\sum_k q(k)-1).$$ Partial
derivatives with respect to $q(k)$ give us the equalities
$$-\phi(k)+{\alpha_k\over
q(k)}=\beta \quad  \forall k\quad  \Rightarrow \quad -\langle
\phi,q\rangle + |\alpha| =\beta.$$ Therefore, the inequalities
for $q^*(k)$ are
\begin{equation}\label{q4}
q^*(k)={\phi(k)q^*(k) - \alpha_k \over \langle \phi,q^*\rangle
-|\alpha|}={\phi(k)q^*(k) - \alpha_k \over \theta -|\alpha|},\quad
\theta:=\langle \phi,q^*\rangle.
\end{equation}
After rewriting (\ref{q4}), we obtain
\begin{equation*}\label{solla2a}
q^*(k)={ \alpha_k \over \phi(k)+|\a|-\theta},\quad k=1,\ldots,K,
\end{equation*}
Thus, there cannot be two solutions having the same $\theta$. As
in the case {\bf 1)}, it is easy to see that when  $\alpha_k>0$
there is only one $\theta$ so that (\ref{solla2a}) sums up to
one. Therefore, the solution to the problem (\ref{probleem2}) is
unique. Note that the solution is independent of $\lambda$.
                                     \end{description}

\subsection*{Acknowledgments} {Research by J{\"u}ri Lember is supported by Estonian Institutional research
funding IUT34-5 and PRG 865. Research by Chris Watkins is supported by grant number (FQXi Grant number FQXi-RFP-IPW-1913) from the Foundational Questions Institute and Fetzer Franklin Fund, a donor advised fund of Silicon Valley Community Foundation.}

\bibliographystyle{plain}

\end{document}